\documentclass{article}

\usepackage[preprint]{neurips_2026}

\usepackage[utf8]{inputenc} 
\usepackage[T1]{fontenc}    
\usepackage{hyperref}       
\usepackage{url}            
\usepackage{booktabs}       
\usepackage{amsfonts}       
\usepackage{nicefrac}       
\usepackage{microtype}      
\usepackage{xcolor}         

\usepackage{amsmath}
\usepackage{amsthm}

\usepackage{paralist}

\title{Structural Preservation and the Logical Expressiveness of Graph Neural Networks}

\author{Przemysław Andrzej Wałęga
\\
Queen Mary University of London
\\
\texttt{p.walega@qmul.ac.uk} \\
\And
Bernardo Cuenca Grau \\
University of Oxford \\
\texttt{bernardo.grau@cs.ox.ac.uk} \\
}

\usepackage{tikz}
\usepackage{cleveref}
\usetikzlibrary{shapes.geometric, positioning, fit, arrows}
\usepackage{todonotes}
\usepackage{thm-restate}
\declaretheorem[name=Theorem]{thm}
\newtheorem{theorem}[thm]{Theorem}

\newtheorem{definition}[thm]{Definition}
\newtheorem{lemma}[thm]{Lemma}
\newtheorem{proposition}[thm]{Proposition}
\newtheorem{corollary}[thm]{Corollary}

\providecommand{\lBrace}{\mathopen{\lbrace\mkern-3mu\vert}}
\providecommand{\rBrace}{\mathclose{\vert\mkern-3mu\rbrace}}

\newcommand{\true}{\ensuremath{\mathsf{true}}}
\newcommand{\false}{\ensuremath{\mathsf{false}}}
\newcommand{\prop}{\ensuremath{\mathsf{PROP}}}

\newcommand{\M}{\ensuremath{\mathfrak{M}}}
\newcommand{\N}{\ensuremath{\mathfrak{N}}}

\newcommand{\ML}{\ensuremath{\mathcal{ML}}}

\newcommand{\EML}{\ensuremath{\exists\mathcal{ML}}}
\newcommand{\EPML}{\ensuremath{\exists^+\mathcal{ML}}}

\newcommand{\GML}{\ensuremath{\mathcal{GML}}}
\newcommand{\EGML}{\ensuremath{\exists\mathcal{GML}}}
\newcommand{\EPGML}{\ensuremath{\exists^+\mathcal{GML}}}

\newcommand{\ALCQ}{\ensuremath{\mathcal{ALCQ}}}
\newcommand{\ELUQ}{\ensuremath{\mathcal{ELUQ}}}

\newcommand{\G}{\ensuremath{G}}

\newcommand{\emb}{\ensuremath{\lambda}}
\newcommand{\unr}{\ensuremath{\mathsf{Unr}}}

\newcommand{\GN}{\ensuremath{\mathcal{N}}}
\newcommand{\C}{\ensuremath{\mathcal{C}}}
\newcommand{\T}{\ensuremath{\mathcal{T}}}
\newcommand{\Lr}[1][]{%
  \ifx#1\empty
    \ensuremath{\mathcal{L}}%
  \else
    \ensuremath{\mathcal{L}^{#1}}%
  \fi
}   
\newcommand{\comb}{\ensuremath{\mathsf{comb}}}
\newcommand{\agg}{\ensuremath{\mathsf{agg}}}

\newcommand{\cls}{\ensuremath{\mathsf{cls}}}

\newcommand{\s}{\ensuremath{\mathsf{s}}}

\newcommand{\ac}{\ensuremath{\mathsf{AC}}}

\newcommand{\GNN}[2]{\ensuremath{\text{GNN}_{#1}^{#2}}}

\newcommand{\Tmin}{\ensuremath{\mathcal{T}_{\mathsf{min}}}}

\begin{document}

\maketitle

\begin{abstract}
Bridges between graph neural networks (GNNs) and logical formalisms have been established by fixing architectural choices, such as the types of aggregation, combination, and activation functions. These choices define restricted classes of GNNs for which tight correspondences with logical formalisms can be obtained, by showing that logical formulae can be translated into equivalent GNNs and, conversely, that GNNs can be translated into equivalent formulae. 

In this paper we take a semantic perspective by establishing the logical expressiveness of classes of GNN classifiers that are preserved under structural  properties: embeddings (extensions), injective homomorphisms, and homomorphisms. We show that, for each such property, there exists a fragment of graded modal logic characterising the class of GNNs. In particular, preservation under embeddings, injective homomorphisms, and homomorphisms corresponds to existential graded modal logic, its existential-positive fragment, and existential-positive modal logic, respectively. These results characterise the expressiveness of broad classes of GNNs independently of specific architectural choices, but we also show that each of these classes admits a GNN architecture of the same expressiveness.

Technically, our approach uses a new well-quasi-order result for trees of bounded height, yielding finite representations of unravelling-invariant classes. 
\end{abstract}

\section{Introduction}

Graph neural networks (GNNs) \cite{DBLP:conf/icml/GilmerSRVD17}
are models designed to operate directly on graphs of variable size while ensuring that model predictions are invariant under graph isomorphisms. 
In the aggregate-combine paradigm, each node maintains a vector representation that is iteratively updated by aggregating information from its neighbours. 
This design induces a strong structural bias: GNNs with $L$ layers are invariant under isomorphisms and inherently local, in the sense that node representations depend only on the node's $L$-hop neighbourhood.
These two properties bring GNNs close to logical formalisms, in particular modal logics, which are also local and isomorphism-invariant \cite{DBLP:conf/iclr/BarceloKM0RS20,hauke2025aggregate,boundedGNNs}.
Establishing such connections 
is beneficial, as logic provides principled tools for analysing expressiveness and enabling symbolic methods for verification and explanation.

These connections have been studied extensively. In particular, bridges between 
GNNs and logic have been 
established by fixing architectural choices, such as 
aggregation, combination, and activation functions.
These choices define restricted families of GNN classifiers for which 
precise correspondences with logical formalisms can be 
obtained. In particular, prior work has identified such 
families corresponding to extensions of first-order logic 
with counting terms \cite{grohe2024descriptive,DBLP:conf/ijcai/NunnSST24} and Presburger quantifiers \cite{benedikt_et_al:LIPIcs.ICALP.2024.127}, as well as 
to rule-based knowledge representation formalisms such as 
Datalog \cite{tena2025expressive,DBLP:conf/kr/CucalaGMK23}.

In this paper, we complement such architecture-driven
characterisations with a semantic framework based on
structural properties that impose strictly stronger constraints on classifiers than invariance under graph isomorphisms, namely preservation under embeddings, injective homomorphisms, and homomorphisms.
We show that 
for each such property, the corresponding class of GNN classifiers has the same expressiveness as a fragment of graded modal logic (\Cref{th:preserve-unified}, \Cref{GNNspreserved}):
\begin{compactitem}
\item The class of GNN classifiers preserved under \textbf{embeddings} has the same expressiveness as existential graded modal logic $\EGML$,
\item The class of  GNN classifiers preserved under \textbf{injective homomorphisms} has the same expressiveness as existential-positive graded modal logic $\EPGML$, and
\item The class of GNN classifiers preserved under \textbf{homomorphisms} has the same expressiveness as existential-positive modal logic $\EPML$.
\end{compactitem}

Our approach is inspired by finite model theory, where preservation theorems relate structural properties to logical definability \cite{DBLP:books/sp/Libkin04,DBLP:journals/bsl/Rosen02,DBLP:series/txtcs/GradelKLMSVVW07,DBLP:journals/jolli/Rosen97,DBLP:journals/apal/AbramskyR24}. Classical results include the van Benthem–Rosen theorem \cite{DBLP:journals/jolli/Rosen97}, which characterises modal logic as the bisimulation-invariant fragment of first-order logic, and Rossman’s theorem \cite{DBLP:journals/jacm/Rossman08}, which characterises existential-positive first-order logic as the fragment preserved under homomorphisms.
We adapt this methodology to the setting of GNNs. Since existing preservation theorems do not directly apply in this setting, we develop new results tailored to graph structures induced by GNN computations.

The key intuition underlying our approach is that, in addition to being invariant under isomorphisms, GNNs are inherently local: in a GNN with $L$ layers, the representation of a node depends only on its $L$-hop neighbourhood \cite{DBLP:conf/iclr/BarceloKM0RS20,DBLP:conf/icml/WangZ22,DBLP:conf/aaai/0001RFHLRG19}. This locality allows us to reduce the analysis of GNNs to trees obtained via unravelling \cite{DBLP:conf/iclr/BarceloKM0RS20}. Combined with preservation under structural relations, it yields the logical characterisations stated above.

The main technical challenge is to obtain a finite representation 
of such classes. To address this, we show a new result on well-quasi-orders (\Cref{embedding_wqo}) which is of independent interest:
\begin{compactitem}
\item Every class of trees of bounded height is \textbf{well-quasi-ordered} by the embedding relation. 
\end{compactitem}

Since $L$-unravellings have bounded height and well-quasi-orders 
admit only finitely many minimal elements, this yields a finite 
representation of each class via its minimal trees. At a high 
level, each minimal tree can be characterised by a logical formula, and the 
entire class is defined by a finite disjunction of such formulae. 
The fragment of modal logic is determined by the preservation relation.

In addition to the semantic characterisations, we show 
that each class of GNN classifiers defined via preservation properties can be captured by a specific architecture. 
In particular  (\Cref{MONMGNN}):
\begin{compactitem}
\item GNNs preserved under \textbf{injective homomorphisms} have the same expressiveness as monotonic GNNs (i.e., GNNs with non-decreasing aggregation and combination functions),
\item GNNs preserved under \textbf{homomorphisms} have the same expressiveness as monotonic GNNs using $\mathrm{MAX}$ as aggregation,
\item GNNs preserved under \textbf{embeddings} have the same expressiveness as monotonic GNNs with an additional  layer that transforms node features independently of their neighbours.
\end{compactitem}


Moreover, our results reveal an interesting property of GNN architectures considered in
prior work, which we refer to as positive-weight GNNs, whose parameter matrices have non-negative entries. In particular, it follows that any GNN preserved under injective homomorphisms (and so, also any monotonic GNN) can be 
transformed into an equivalent positive-weight GNN. Thus, with respect to expressiveness, the syntactic restriction to non-negative weights is not restrictive within this class.

In summary, our contributions are as follows. We develop a 
framework for analysing the expressiveness of GNNs 
based on structural preservation properties. We establish 
preservation theorems linking these properties to definability in fragments of graded modal logic, supported by new well-quasi-ordering results ensuring finite representations. Finally, 
we provide  architectural characterisations of the GNN classes defined by preservation properties. 

\section{Related Work}


The expressiveness of GNNs has been studied from several perspectives. Their \emph{discriminative power} measures their ability to distinguish graphs, and it is well known that message-passing GNNs are bounded by
the Weisfeiler–Lehman (WL) graph isomorphism test. Thus,
they cannot distinguish graphs that WL fails to separate 
\cite{morris2019weisfeiler,DBLP:conf/iclr/XuHLJ19}. 
Work on the logical expressiveness of GNNs treats them as node classifiers that compute unary queries over graphs and establishes correspondences between GNN architectures and 
logical formalisms. In particular, GNN expressiveness is not captured by first-order logic alone and requires extensions with counting terms, 
Presburger quantifiers, or related mechanisms \cite{grohe2024descriptive,benedikt_et_al:LIPIcs.ICALP.2024.127,DBLP:conf/ijcai/NunnSST24}. Furthermore, connections to graded modal logic and 
related fragments have been established 
\cite{DBLP:conf/iclr/BarceloKM0RS20}. Recent work on bounded 
GNNs refines these correspondences by relating aggregation schemes with 
bounded multiplicities to two-variable fragments of first-order logic \cite{boundedGNNs}.


A complementary line of work studies connections between GNNs
and rule-based knowledge representation formalisms motivated by logical reasoning and explainability. 
In particular, monotonic GNNs have been introduced to capture rule-based inference in Datalog \cite{tena2025expressive,DBLP:conf/kr/CucalaGMK23,DBLP:conf/kr/CucalaG24,morris2025sound,DBLP:conf/kr/MorrisCG024}. 
These models enforce structural preservation properties, such as preservation under homomorphisms or injective homomorphisms. These properties are enforced through architectural constraints, including non-negative weights, monotone activation functions, and restricted aggregation.
While effective in practice, these approaches are inherently syntax-driven, as the desired preservation properties are enforced indirectly through design choices rather than characterised explicitly in semantic terms.


Our work is also related to preservation theorems in finite model theory, which characterise logical fragments in terms of structural preservation properties and, in some cases, invariance conditions. 
Classical results such as Łoś-Tarski, Lyndon, and Homomorphism Preservation theorems establish correspondences between fragments of first-order logic and preservation under embeddings, surjective homomorphisms, and homomorphisms, respectively \cite{los1955extending,lyndon1959properties,DBLP:books/daglib/0030198}. 
While many such results fail over finite structures (and thus over graphs), Rossman’s theorem shows that homomorphism preservation continues to hold in the finite setting \cite{DBLP:journals/jacm/Rossman08}. 
In modal logic, preservation results are tied to locality and invariance under bisimulation. 
Van Benthem’s theorem (and its finite variant due to Rosen) characterises bisimulation-invariant first-order properties as those definable in modal logic  \cite{DBLP:journals/jphil/AndrekaNB98,DBLP:journals/jolli/Rosen97}. More recent work has extended preservation results to graded modal logic and related settings \cite{DBLP:journals/apal/AbramskyR24}. 
Our results on well-quasi-orders are related to recent work in a similar direction \cite{DBLP:phd/hal/Lopez23a}, however, these techniques rely on additional structural assumptions, such as closure under substructures, that do not hold in our setting.

\section{Background}

\paragraph{Graphs and Graph Relations}
We consider finite, undirected, simple, node-labelled graphs $\G = (V, E, \emb)$, where $V$ is a finite set of nodes, $E$ a set of undirected edges without self-loops, and $\emb:V \to \{0,1 \}^d$ assigns to each node a binary vector of fixed dimension $d$. 
Intuitively, $d$ represents the number of possible labels (propositional features) and  $\emb_i(v) = 1$, for a label $i \in [d]$,  indicates that node $v$ has label $i$;  the set $\{ 1, \dots, n\}$
is represented as $[n]$.
We write vectors in bold, e.g., $\mathbf{x}$, $\mathbf{y}$, $\mathbf{z}$, and use $x_i$ to denote the $i$-th component of 
$\mathbf{x}$. 
For vectors $\mathbf{x},\mathbf{y} \in \mathbb{R}^d$, we write $\mathbf{x} \leq \mathbf{y}$ if $x_i \leq y_i$ for all $i \in [d]$. 
This discrete labelling enables a direct connection with logical formalisms \cite{DBLP:conf/iclr/BarceloKM0RS20,benedikt_et_al:LIPIcs.ICALP.2024.127}.

A \emph{pointed graph} is a pair $(\G,v)$ consisting of a graph and a distinguished node.
A \emph{tree} is a connected acyclic graph. A \emph{rooted tree} is a pointed graph $(\G, v)$ where $\G$ is a tree and $v$ is its root. 
The \emph{depth} of a node $u$ in a pointed tree is the 
length of the unique path from the root to $u$, and the \emph{height} of a tree $\G$ is the maximum depth of any node.

Let $(\G, w)$ and $(H, v)$ be pointed graphs, with
$\G = (V, E, \emb)$ and $H = (V', E', \emb')$. 
We will consider the following definitions of mappings $f: V \rightarrow V'$, with
$f(w) = v$.

\begin{itemize} 
\item \emph{Isomorphism:} $f$ is a bijection, $\emb(u)=\emb'(f(u))$, and $\{u,z\} \in E$ iff  $\{f(u),f(z)\} \in E'$.


\item \emph{Embedding:} $f$ is an isomorphism from $(\G,w)$ onto the subgraph of $H$ induced by $f(V)$.

\item \emph{Homomorphism:} $f$ satisfies $\emb(u) \leq \emb'(f(u))$ for all $u \in V$, and $\{u,z\} \in E$ implies $\{f(u),f(z)\} \in E'$. If $f$ is also injective, we call it an \emph{injective homomorphism}.

\end{itemize}

Intuitively, these mappings capture different ways in which one graph can be embedded into another: isomorphisms preserve the 
entire structure, embeddings preserve a subgraph exactly, while homomorphisms allow labels to be extended and nodes to be merged (unless the mapping is injective).
For examples of these relations see \Cref{fig:graph-relations}.
Clearly, every isomorphism is an embedding, every embedding is an injective homomorphism, and every injective homomorphism is a homomorphism.

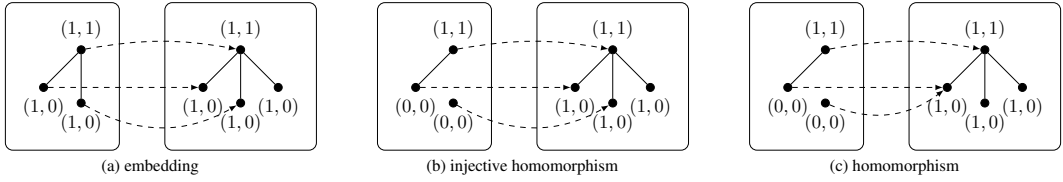
\begin{figure}[ht]
\centering
\resizebox{\linewidth}{!}{%
\begin{tikzpicture}
\tikzset{
  >=latex,
  node distance=1.0cm,
  world/.style={
    draw, circle, fill,
    minimum size=1.5mm,
    inner sep=0pt,
    scale=1
  }
}

\begin{scope}[local bounding box=G1]
  \node[world, label=above:${(1,1)}$] (g1w) at (0,0) {};
  \node[world, below left  of=g1w, label=below:${(1,0)}$]  (g1v1) {};
  \node[world, below  of=g1w, label=below:${(1,0)}$] (g1v2) {};
  \draw[-] (g1w) -- (g1v1);
  \draw[-] (g1w) -- (g1v2);
\end{scope}
\node[draw, rounded corners, fit=(G1), inner sep=6pt] (fG1) {};

\begin{scope}[xshift=3cm, local bounding box=H1]
  \node[world, label=above:${(1,1)}$] (h1u) at (0,0) {};
  \node[world, below left  of=h1u, label=below:${(1,0)}$]  (h1u1) {};
  \node[world, below        of=h1u, label=below:${(1,0)}$] (h1ux) {};
  \node[world, below right of=h1u, label=below:${(1,0)}$] (h1u2) {};
  \draw[-] (h1u) -- (h1u1);
  \draw[-] (h1u) -- (h1u2);
  \draw[-] (h1u) -- (h1ux);
\end{scope}
\node[draw, rounded corners, fit=(H1), inner sep=6pt] (fH1) {};

\draw[->, dashed, bend left=10]  (g1w)  to (h1u);
\draw[->, dashed]                (g1v1) to (h1u1);
\draw[->, dashed, bend right=30] (g1v2) to (h1ux);

\node at (1.3, -2.2) {\small (a)~embedding};

\begin{scope}[xshift=7cm, local bounding box=G2]
  \node[world, label=above:${(1,1)}$] (g2w) at (0,0) {};
  \node[world, below left  of=g2w, label=below:${(0,0)}$]  (g2v1) {};
  \node[world, below  of=g2w, label=below:${(0,0)}$] (g2v2) {};
  \draw[-] (g2w) -- (g2v1);
\end{scope}
\node[draw, rounded corners, fit=(G2), inner sep=6pt] (fG2) {};
\begin{scope}[xshift=10cm, local bounding box=H2]
  \node[world, label=above:${(1,1)}$] (h2u) at (0,0) {};
  \node[world, below left  of=h2u, label=below:${(1,0)}$]  (h2u1) {};
  \node[world, below        of=h2u, label=below:${(1,0)}$] (h2ux) {};
  \node[world, below right of=h2u, label=below:${(1,0)}$] (h2u2) {};
  \draw[-] (h2u) -- (h2u1);
  \draw[-] (h2u) -- (h2u2);
  \draw[-] (h2u) -- (h2ux);
\end{scope}
\node[draw, rounded corners, fit=(H2), inner sep=6pt] (fH2) {};
\draw[->, dashed, bend left=10]  (g2w)  to (h2u);
\draw[->, dashed]                (g2v1) to (h2u1);
\draw[->, dashed, bend right=30] (g2v2) to (h2ux);
\node at (8.3, -2.2) {\small (b)~injective homomorphism};

\begin{scope}[xshift=14cm, local bounding box=G3]
  \node[world, label=above:${(1,1)}$] (g3w) at (0,0) {};
  \node[world, below left  of=g3w, label=below:${(0,0)}$]  (g3v1) {};
  \node[world, below  of=g3w, label=below:${(0,0)}$] (g3v2) {};
  \draw[-] (g3w) -- (g3v1);
\end{scope}
\node[draw, rounded corners, fit=(G3), inner sep=6pt] (fG3) {};
\begin{scope}[xshift=17cm, local bounding box=H3]
  \node[world, label=above:${(1,1)}$] (h3u) at (0,0) {};
  \node[world, below left  of=h3u, label=below:${(1,0)}$]  (h3u1) {};
  \node[world, below        of=h3u, label=below:${(1,0)}$] (h3ux) {};
  \node[world, below right of=h3u, label=below:${(1,0)}$] (h3u2) {};
  \draw[-] (h3u) -- (h3u1);
  \draw[-] (h3u) -- (h3u2);
  \draw[-] (h3u) -- (h3ux);
\end{scope}
\node[draw, rounded corners, fit=(H3), inner sep=6pt] (fH3) {};
\draw[->, dashed, bend left=10]  (g3w)  to (h3u);
\draw[->, dashed]                (g3v1) to (h3u1);
\draw[->, dashed, bend right=30] (g3v2) to (h3u1);
\node at (15.3, -2.2) {\small (c)~homomorphism};

\end{tikzpicture}%
}
\caption{Embedding (a), injective homomorphism (not an embedding) (b), and a homomorphism which is not injective (c)}
\label{fig:graph-relations}
\end{figure}

Each notion above induces a corresponding binary relation $\preceq$ on pointed graphs, such that $(\G,w) \preceq (H,v)$  if there exists a mapping of the corresponding type from $(\G,w)$ to $(H,v)$.

Let $\C$ be a class of pointed graphs and $\preceq$ a binary
relation as above.
Class $\C$ is \emph{preserved} under $\preceq$ if $(\G,w) \in \C$ and $(\G,w) \preceq (H,v)$ implies that $(H,v) \in \C$.
Class $\C$  is \emph{invariant} under $\preceq$ if $\C$ is preserved under both $\preceq$ and its inverse relation $\preceq^{-1}$.
These preservation properties capture the robustness of classifiers under structural transformations of the input graph.

\paragraph{Graph Neural Networks}
We consider node classification with graph neural networks (GNNs) based on the standard \emph{aggregate-combine} paradigm \cite{benedikt_et_al:LIPIcs.ICALP.2024.127,DBLP:conf/iclr/BarceloKM0RS20}. GNNs iteratively update node representations by aggregating information from neighbours and combining it with the current representation.

Formally, a GNN layer is a pair $(\agg, \comb)$, where $\agg$ is an \emph{aggregation function} mapping multisets of vectors into vectors and $\comb$ is a \emph{combination function} mapping  pairs of vectors (or their concatenation) to  vectors.
Applying a layer to $\G = ( V, E, \emb )$ produces a graph $\G' = ( V, E, \emb' )$ with the same nodes and edges, but with updated node features.

For each node $v \in V$, the updated label is defined by 
\begin{equation}
\emb'(v) = \comb \Big( \emb(v), \agg( \lBrace  \emb(w) \rBrace_{w \in N_{\G}(v)} )  \Big),  \label{eq:ac}
\end{equation}
where $N_\G(v) = \{w \mid \{v,w\} \in E \}$ denotes the set of neighbours of $v$, and $\lBrace \cdot \rBrace$ is a multiset.
Aggregation and combination functions have fixed input and output dimensions and must be compatible;
if  $\agg$ maps vectors of dimension $d$ to dimension $d'$, then 
$\comb$ takes as input the concatenation of vectors of dimensions $d$ and $d'$.

Often $\comb$ is a one-layer MLP with a componentwise  activation  $\sigma$ (e.g., ReLU or truncated ReLU):
\begin{equation} \label{eq:simple}
\emb'(v) = \sigma \Big(
\mathbf{b} +
     \emb(v) \mathbf{A}  +
         \agg \big( \lBrace  \emb(w) \mid w \in N_{\G}(v) \rBrace \big) \mathbf{C}  \Big),
\end{equation}
where 
$\mathbf{b}$ is a bias vector, and $\mathbf{A}$ and $\mathbf{C}$ are real-valued matrices. 
Following previous work \cite{DBLP:conf/iclr/BarceloKM0RS20}, we  refer to such layers as \emph{simple}.
Common choices for $\agg$ include position-wise $\mathrm{SUM}$, $\mathrm{MAX}$, and $\mathrm{MEAN}$ although $\mathrm{MEAN}$ is not monotone with respect to the multiset order used later.
Another aggregation is 
$\mathrm{MAX}$-$k$-$\mathrm{SUM}$ \cite{DBLP:conf/kr/CucalaGMK23} 
which sums the $k$ largest values per component.

A \emph{GNN classifier} $\GN$ of dimension $d$ consists of $L$ aggregate-combine layers followed by a 
classification function $\cls$ mapping vectors to truth values. 
In this paper, we consider only threshold-based classification functions satisfying $\cls(x) = \true$ if $x \geq t$ for some fixed threshold $t$ and $\false$ otherwise. 
This restriction is without loss of generality for the expressiveness results proved in this paper.
The output dimension of each layer matches the input dimension of the next.
We write $\emb^{(\ell)}(v)$ for the vector associated with node $v$ after applying layer $\ell$, with  $\emb^{(0)}(v)=\emb(v)$ and $\emb^{(L)}(v)$ the final representation.
The output of $\GN$ on $(\G,v)$ is 
the truth value  $\GN(\G,v) = \cls\bigl(\emb^{(L)}(v)\bigr)$.


\paragraph{Logical Expressiveness}
A \emph{node classifier} is a function mapping
pointed graphs  to truth values $\true$ (accepted) or $\false$ (rejected). 
Two classifiers are \emph{equivalent} if they accept exactly the same pointed graphs.
We compare the expressiveness of families of classifiers arising from GNNs or logical formulae. 
We say that two families of classifiers have the same expressiveness if each classifier in one family has an equivalent classifier in the other, and vice versa.

\section{A Logical Toolkit for Analysing Expressiveness}\label{sec:toolkit}

We introduce logical and combinatorial tools that underpin our analysis.
Our approach is based on relating structural preservation properties to definability in fragments of graded modal logic.

\subsection{Locality and Unravelling}

GNNs are inherently local: after $L$ layers, a node representation depends only on its $L$-hop neighbourhood. This locality can be captured by the notion of \emph{unravelling},
which transforms a pointed graph $(\G,v)$ into a rooted tree $\unr^L(\G,v)$ whose nodes correspond to paths in $\G$  starting at $v$, of length at most $L$.
The $L$-unravelling unfolds all paths of length up to $L$ into a tree of height at most $L$. 
This construction mirrors the computation of a GNN, where each node aggregates information from its neighbourhood in a tree-shaped manner.
An example of a pointed graph and its $2$-unravelling is shown in \Cref{fig:unravelling}. 
Note that the unravelling unfolds cycles into distinct paths, thereby yielding a tree.

\begin{figure}[ht]
\centering
\scalebox{0.7}{
\begin{tikzpicture}
\tikzset{
>=latex,
node distance=1.0cm,
world/.style={
    draw,
    circle,
    fill,
    minimum size=1.5mm,
    inner sep=0pt,
    scale=1
  }
}

\begin{scope}[local bounding box=M1]
\node[world, label=right:$\mathbf{v}_0$] (v1) at (0,0) {};
\node[world, below left of=v1, label=right:$\mathbf{v}_1$] (v2)  {};
\node[world, below of=v1, label=right:$\mathbf{v}_2$] (v3)  {};
\node[world, below right of=v1, label=right:$\mathbf{v}_3$] (v4)  {};
\draw[-] (v1) -- (v2);
\draw[-] (v1) -- (v4);
\draw[-] (v1) -- (v3);
\draw[-] (v2) -- (v3);
\end{scope}
\node[draw, rounded corners, fit=(M1), inner sep=6pt] (frame1) {};
\node[above=2pt of frame1] {$(\G,v)$};

\begin{scope}[xshift=5cm, local bounding box=M2]
\node[world, label=right:$\mathbf{v}_0$] (v1) at (0,0) {};
\node[world, below left of=v1, label=right:$\mathbf{v}_1$] (v2)  {};

\node[world, below of=v1, label=right:$\mathbf{v}_2$] (v3)  {};
\node[world, below right of=v1, label=right:$\mathbf{v}_3$] (v4)  {};

\node[world, below of=v2, label=right:$\mathbf{v}_2$] (u)  {};
\node[world, below left of=v2, label=right:$\mathbf{v}_0$] (u2)  {};

\node[world, below of=v3, label=right:$\mathbf{v}_0$] (w)  {};
\node[world, below right of=v3, label=right:$\mathbf{v}_1$] (w2)  {};

\node[world, below right of=v4, label=right:$\mathbf{v}_0$] (z)  {};

\draw[-] (v1) -- (v2);
\draw[-] (v1) -- (v4);
\draw[-] (v1) -- (v3);

\draw[-] (v2) -- (u);
\draw[-] (v2) -- (u2);
\draw[-] (v3) -- (w);
\draw[-] (v3) -- (w2);

\draw[-] (v4) -- (z);

\end{scope}
\node[draw, rounded corners, fit=(M2), inner sep=6pt] (frame2) {};
\node[above=2pt of frame2] {$\unr^2(\G,v)$};
\end{tikzpicture}
}
\caption{A pointed graph and its $2$-unravelling}
\label{fig:unravelling}
\end{figure}
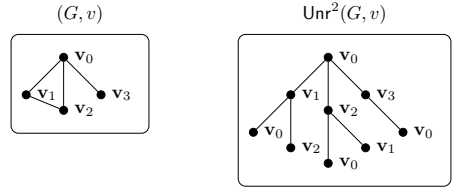
Importantly, GNNs with $L$ layers are invariant under this transformation: evaluating a GNN with $L$ layers on a graph yields the same result as evaluating it on its $L$-unravelling. This observation will be central to our analysis, as it allows us to reason about GNNs using tree-shaped structures, and has been exploited in prior work to relate GNNs to logical formalisms 
\cite{DBLP:conf/iclr/BarceloKM0RS20}.

We now give the formal definition of unravelling.

\begin{definition}
Let $(\G, v)$ be a pointed graph with $\G = (V, E, \emb)$, and let $L \in \mathbb{N}$. 
The $L$-\emph{unravelling} of $(\G, v)$, denoted $\unr^L(\G, v)$, is the pointed graph $((V', E', \emb'), v')$ defined as follows:
\begin{compactitem}
\item $V'$ contains a node for each sequence $(v, v_1, \ldots, v_\ell)$ of length $\ell \leq L$ such that $\{v_i, v_{i+1}\} \in E$ for all $i < \ell$; 
\item  $v'$ corresponds to the sequence $(v)$ of length $0$; 
\item $E'$ has an edge  between $(v, v_1, \ldots, v_{\ell-1})$ and $(v, v_1, \ldots, v_{\ell})$ iff $\{v_{\ell-1}, v_{\ell}\} \in E$; 
\item $\emb'$ is defined by $\emb'((v, v_1, \ldots, v_{\ell})) = \emb(v_{\ell})$.
\end{compactitem}
\end{definition}

Thus, $\unr^L(\G, v)$ is a tree of height at most $L$, whose nodes represent all paths of length at most $L$ starting from $v$.
This construction captures the locality of GNN computations: GNN classifiers are invariant under unravelling, as formalised next (see also \cite[Proposition C.4]{DBLP:conf/iclr/BarceloKM0RS20}).

\begin{proposition}
Let $\GN$ be an aggregate-combine GNN with $L$ layers. Then, for every pointed graph $(\G,v)$ and every $K \geq L$, we have
$
\GN(\G,v) = \GN(\unr^{K}(\G,v)).
$
\end{proposition}

Modal logics exhibit an analogous invariance property: the truth of a modal formula of depth $L$ at a node depends only on its $L$-unravelling.

\subsection{Modal Logic}

Formulae of graded modal logic $\GML$ \cite{DBLP:journals/sLogica/Rijke00,DBLP:journals/logcom/Tobies01,DBLP:journals/corr/abs-1910-00039}  are defined inductively by the grammar
$$
\varphi :=  p \mid \neg \varphi \mid \varphi \land \varphi \mid \varphi \lor \varphi \mid \Diamond^{\ge k} \varphi,
$$
where $p$ ranges over propositions $p_i$ and $k \in \mathbb{N}$. The \emph{depth} of $\varphi$ is the maximum nesting depth of its modal operators. The \emph{existential} fragment $\EGML$ restricts negation to occur in front of propositions, while the \emph{existential-positive} fragment $\EPGML$ is negation-free.
The logic \ML{} restricts the graded modal operators to $k=1$, and the fragments \EML{} and \EPML{} of \ML{} are defined analogously.

We interpret formulae over pointed graphs\footnote{Modal formulae are typically interpreted over Kripke structures (directed graphs); in our setting, we work with undirected graphs by
viewing the edge relation as symmetric.} $(\G,v)$ with $\G = (V,E,\emb)$. We assume that each proposition $p_i$ corresponds to the $i$-th 
component of the node feature vector.

The \emph{satisfaction relation} $\models$ is defined inductively as follows:
\begin{align*}
& (\G, v) \models p_i 
&& \text{iff} && \emb(v)_i = 1
\\
& (\G, v) \models \neg \varphi 
&& \text{iff} && (\G, v) \not\models \varphi
\\
& (\G, v) \models \varphi \land \psi 
&& \text{iff} && (\G, v) \models \varphi \text{ and } (\G, v) \models \psi
\\
& (\G, v) \models \varphi \lor \psi 
&& \text{iff} && (\G, v) \models \varphi \text{ or } (\G, v) \models \psi
\\
& (\G, v) \models \Diamond^{\ge k} \varphi 
&& \text{iff} && \text{there exist at least $k$ neighbours } w \text{ of } v  \text{ such that } (\G, w) \models \varphi 
\end{align*}

For example, the formula $p_1 \land \Diamond^{\geq 2} p_2$ holds in a pointed graph $(\G,v)$ if $v$ satisfies $p_1$ (i.e., $\emb(v)_1 = 1$) and it has at least two neighbours satisfying $p_2$ (i.e., their second component is
$1$). This formula belongs to $\EGML$ (and even to $\EPGML$), but is not expressible in $\ML$.


It is well known that for each $\GML$ formula $\varphi$ of depth $L$ and each pointed graph $(\G,v)$, we have
\[
(\G,v) \models \varphi \quad \text{iff} \quad \unr^L(\G,v) \models \varphi.
\]
That is, the class of pointed graphs satisfying $\varphi$ is invariant under $L$-unravelling.

Given a modal logic $\mathcal{L}$, each formula $\varphi \in \mathcal{L}$ defines a node classifier by mapping a pointed graph $(\G,v)$ to $\true$ if $(\G,v) \models \varphi$, and to $\false$ otherwise. We refer to such classifiers as \emph{$\mathcal{L}$-classifiers}.

A class $\mathcal{C}$ of pointed graphs is \emph{definable} in a modal logic $\mathcal{L}$ if there exists a formula $\varphi \in \mathcal{L}$ such that $(\G,v) \models \varphi$ if and only if $(\G,v) \in \mathcal{C}$.

\subsection{Well-quasi-orders}

A key step in our approach is to ensure that only finitely many structures need to be considered when 
characterising a class of graphs. This finiteness property is guaranteed by well-quasi-orders (wqos).

A \emph{quasi-order} (or \emph{preorder}) $(A, \preceq)$ is a set $A$ equipped with a binary 
relation $\preceq$ that is reflexive ($a \preceq a$ for all $a \in A$) and transitive 
($a \preceq b$ and $b \preceq c$ implies $a \preceq c$).
A quasi-order $(A,\preceq)$ is a \emph{well-quasi-order} if every infinite sequence $a_1,a_2,\ldots$ in $A$ contains 
indices $i<j$ such that $a_i \preceq a_j$. Intuitively, this rules out infinite antichains and ensures that every infinite sequence contains an increasing pair. In particular, any subset of a wqo has only finitely many $\preceq$-minimal elements (up to equivalence), which will be crucial for our results.

A central tool in our analysis is Higman’s Lemma, which ensures that the wqo property is preserved when extended to sequences.
Let $A^{*}$ be the set of all finite sequences  over $A$, and 
let $\preceq^{*}$ be the \emph{subsequence order}, defined by $(a_1,\ldots,a_n) \preceq^{*} (b_1,\ldots,b_m)$ if there exists a strictly increasing map 
$f:\{1,\ldots,n\} \to \{1,\ldots,m\}$ such that $a_i \preceq b_{f(i)}$ for all $i \in \{1,\ldots,n\}$.
For example, $(a_1, a_2) \preceq^{*} (b_1, b_2, b_3)$ holds if $a_1 \preceq b_1$ 
and $a_2 \preceq b_3$.

\begin{lemma}[Higman's Lemma]\label{higman}
If $(A,\preceq)$ is a wqo, then $(A^{*},\preceq^{*})$ is also a wqo.
\end{lemma}

\subsection{A New Well-Quasi-Ordering Result for Trees}
\label{sec:new_wqo}

We now apply the wqo framework to trees, in particular to those obtained via unravelling, although the results of this section apply more generally. 
We show that any (possibly infinite) class of trees obtained by $L$-unravelling has finitely many minimal trees with respect to embedding. This finiteness is crucial, as it allows us to represent an infinite class using only finitely many minimal elements.
To establish this, we prove a stronger result: the embedding relation on any class of trees of bounded height is a wqo. Since $L$-unravellings have height at most $L$, and every wqo admits finitely many 
minimal elements, this result yields the desired finiteness.

This is the main technical result underlying our framework. Beyond its role in establishing our preservation 
theorems, it is of independent interest in the study of well-quasi-orders.

Intuitively, trees of bounded height can be decomposed into their immediate subtrees, which themselves have smaller height. This recursive structure 
allows one to reduce embeddings between trees to embeddings between sequences. Using Higman’s Lemma we can 
ensure that such sequences form a well-quasi-order, yielding the result.

\begin{restatable}{theorem}{embeddingwqo}\label{embedding_wqo}
The embedding relation on any class of rooted trees, whose heights are uniformly bounded by a fixed constant,
is a well-quasi-order.
\end{restatable}

As a consequence, the same result holds for injective homomorphisms and homomorphisms, since these relations extend the embedding relation.

\begin{corollary}\label{cor:finiteness}
The injective homomorphism and homomorphism relations on any class of rooted trees of bounded height are well-quasi-orders.
\end{corollary}

Any such class admits finitely many minimal trees, which provide a finite 
representation of the entire (infinite) class. This property underpins the preservation theorems developed in the next section.
The bounded-height assumption is essential: without it, the wqo property fails.

\begin{restatable}{theorem}{unboundedbad}
\label{unbounded_bad}
The embedding, injective homomorphism, and homomorphism relations on classes of rooted trees of 
unbounded height are not, in general, well-quasi-orders.
\end{restatable}

For example, consider the case of  injective homomorphisms and the class $\T$ of rooted trees shown in \Cref{antichain}, where nodes are unlabelled. 
For any $ n \neq m$, 
there is neither an injective homomorphism from $(\G_n, v_1)$ to $(\G_m, v_1)$, nor vice-versa.
Intuitively, the position of the branching nodes depends on the height of the tree, making these structures incomparable under injective homomorphisms.
Thus, $\T$ forms an infinite antichain with respect to injective homomorphisms, and so, it is not a wqo.

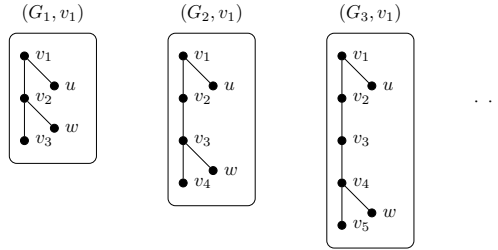
\begin{figure}[ht]
\centering
\scalebox{0.7}{
\begin{tikzpicture}
\tikzset{
>=latex,
node distance=0.8cm,
world/.style={
    draw,
    circle,
    fill,
    minimum size=1.5mm,
    inner sep=0pt,
    scale=1
  }
}

\begin{scope}[local bounding box=M1]
\node[world, label=right:$v_1$] (v1) at (0,0) {};
\node[world, below of=v1, label=right:$v_2$] (v2)  {};
\node[world, below of=v2, label=right:$v_3$] (v3)  {};

\node[world, below right of=v1, label=right:$u$] (u)  {};
\node[world, below right of=v2, label=right:$w$] (w)  {};
\draw [-] (v1) --  (v2);
\draw [-] (v2) --  (v3);

\draw [-] (v1) --  (u);
\draw [-] (v2) --  (w);
\end{scope}
\node[draw, rounded corners, fit=(M1), inner sep=6pt] (frame1) {};
\node[above=2pt of frame1] {$(G_1,v_1)$};

\begin{scope}[xshift=3cm, local bounding box=M2]
\node[world, label=right:$v_1$] (v1) at (0,0) {};
\node[world, below of=v1, label=right:$v_2$] (v2)  {};
\node[world, below of=v2, label=right:$v_3$] (v3)  {};
\node[world, below of=v3, label=right:$v_4$] (v4)  {};

\node[world, below right of=v1, label=right:$u$] (u)  {};
\node[world, below right of=v3, label=right:$w$] (w)  {};
\draw [-] (v1) --  (v2);
\draw [-] (v2) --  (v3);
\draw [-] (v3) --  (v4);

\draw [-] (v1) --  (u);
\draw [-] (v3) --  (w);
\end{scope}
\node[draw, rounded corners, fit=(M2), inner sep=6pt] (frame2) {};
\node[above=2pt of frame2] {$(G_2,v_1)$};

\begin{scope}[xshift=6cm, local bounding box=M3]
\node[world, label=right:$v_1$] (v1) at (0,0) {};
\node[world, below of=v1, label=right:$v_2$] (v2)  {};
\node[world, below of=v2, label=right:$v_3$] (v3)  {};
\node[world, below of=v3, label=right:$v_4$] (v4)  {};
\node[world, below of=v4, label=right:$v_5$] (v5)  {};

\node[world, below right of=v1, label=right:$u$] (u)  {};
\node[world, below right of=v4, label=right:$w$] (w)  {};
\draw [-] (v1) --  (v2);
\draw [-] (v2) --  (v3);
\draw [-] (v3) --  (v4);
\draw [-] (v4) --  (v5);

\draw [-] (v1) --  (u);
\draw [-] (v4) --  (w);
\end{scope}
\node[draw, rounded corners, fit=(M3), inner sep=6pt] (frame3) {};
\node[above=2pt of frame3] {$(G_3,v_1)$};

\node[above right=-1.5cm and 1cm of frame3] {$\cdot\ \cdot\ \cdot$};

\end{tikzpicture}
}
\caption{An infinite antichain of rooted trees with respect to the injective homomorphism 
}\label{antichain}
\end{figure}

\subsection{Preservation Theorems}\label{sec:preservation}

We now combine the well-quasi-ordering results from the previous section with unravelling invariance to obtain preservation theorems.

At a high level, the idea is as follows. Let $\C$ be a class of pointed graphs invariant under $L$-unravelling and preserved under a structural relation $\preceq$ (such as embedding, injective homomorphism, or 
homomorphism). Then membership of a pointed graph in $\C$ is determined by its $L$-unravelling, which is a tree of bounded height.
By the wqo results, such trees admit a finite set of $\preceq$-minimal elements.
This finiteness allows us to represent $\C$ using a finite logical description: each minimal tree can be characterised by a formula, 
and the class $\C$ is defined by a finite disjunction of such formulae. 

Each of the relations $\preceq$ corresponds to a fragment of modal logic: embeddings, injective homomorphisms, and homomorphisms correspond to $\EGML$, $\EPGML$, and $\EPML$, respectively. We formalise this in the following 
theorem.

\begin{restatable}{theorem}{thpreserveunified}\label{th:preserve-unified}
Let $\C$ be a class of pointed graphs, $L \in \mathbb{N}$, and let $\preceq$ be one of the following relations:
embedding, injective homomorphism, or homomorphism.
Then the following are equivalent: 
\begin{compactitem}
\item  $\C$ is invariant under $L$-unravelling and preserved under $\preceq$ 
\item $\C$ is definable by a formula of depth at most $L$ in a corresponding logic: $\EGML$ if $\preceq$ is embedding, $\EPGML$ if $\preceq$ is injective homomorphism, and $\EPML$ if $\preceq$ is homomorphism.
\end{compactitem}
\end{restatable}

Since $L$-layer GNN classifiers are invariant under $L$-unravelling, these preservation theorems yield a semantic characterisation of GNN expressiveness.

\begin{corollary}
\label{GNNspreserved}
Let $\GN$ be a
GNN with $L$ layers, let $\C$ be the class of pointed graphs accepted by $\GN$, and let $\preceq$ be one of the  relations:
embedding, injective homomorphism, or homomorphism. Then the following are equivalent:
\begin{compactitem}
\item $\C$ is preserved under $\preceq$ 
\item  $\C$ is definable by a formula of depth at most $L$ in the corresponding fragment of graded modal logic: $\EGML$ if $\preceq$ is embedding, $\EPGML$ if $\preceq$ is injective homomorphism, and $\EPML$ if $\preceq$ is homomorphism.
\end{compactitem}
\end{corollary}

\section{Syntactic Characterisations of Structural Preservation Classes}\label{sec:MGNN}


We now complement the semantic characterisations with corresponding 
syntactic ones, showing that each preservation class identified in the
previous section can be realised by a suitable class of GNN architectures.
In particular, we show that GNN classifiers preserved under 
embeddings can be realised as monotonic GNNs
preceded by a pointwise feature transformation. Removing this 
preprocessing step yields a class of monotonic GNNs corresponding
to preservation under injective homomorphisms, which generalises 
previously proposed monotonic architectures. Further 
restricting the aggregation functions yields GNN classes 
corresponding to preservation under homomorphisms.

These architectures consist of two types of layers: monotonic layers which involve only non-decreasing combination and aggregation functions, and augmentation layers which
transform initial node feature vectors into non-negative vectors independently of the graph structure.

\begin{definition}
For multisets $M$ and $M'$ of vectors of the same dimension, we write $M \leq M'$  if there exists an injective mapping from the elements of $M$ to the elements of $M'$  such that $\mathbf{x} \leq f(\mathbf{x})$, for every $\mathbf{x} \in M$.  A function $f$ mapping vectors or multisets to vectors is \emph{non-decreasing} if $x \leq y$ implies $f(x) \leq f(y)$ for all inputs $x,y$ of the same type (e.g. vectors or multisets).
\end{definition}

Using this notion, we define monotonic and augmentation layers and the relevant GNN architectures.

\begin{definition}
A GNN layer $(\agg, \comb)$ is \emph{monotonic} if both the aggregation  $\agg$ and the combination  $\comb$ functions are non-decreasing. 
An \emph{augmentation layer} is a function $\eta : \mathbb{R}^d \to \mathbb{R}^{d'}_{\geq 0}$, applied pointwise to node feature vectors, where $\mathbb{R}_{\geq 0}$ are non-negative real numbers.
\end{definition}

\begin{definition}
A \emph{monotonic GNN} (MGNN) is a GNN whose layers are all monotonic. 
A $\mathrm{MAX}$-MGNN is an MGNN using only $\mathrm{MAX}$ as aggregation.
Then an \emph{augmented ($\mathrm{MAX}$-)MGNN} is an ($\mathrm{MAX}$-)MGNN with an additional augmentation layer before  monotonic layers.
\end{definition}

Monotonic GNNs have been studied previously in the context of explainability and logical reasoning, where architectural constraints
are imposed to enforce structural preservation properties \cite{tena2025expressive,DBLP:conf/kr/CucalaGMK23}. These models are designed to satisfy monotonicity 
conditions that ensure preservation under homomorphisms or
injective homomorphisms. Our definitions subsume these approaches by isolating the underlying semantic property and expressing it in a 
modular, layer-wise manner. In particular, the class of monotonic 
GNNs defined above generalises the architectures proposed in prior
work, which we recapitulate in the following definition.
\begin{definition}
A \emph{positive-weight GNN} is a GNN consisting of simple layers (see Equation \eqref{eq:simple}) in which all weight matrices have non-negative entries,
the activation function is non-decreasing, and the aggregation 
functions are  $\mathrm{SUM}$, $\mathrm{MAX}$, or 
$\mathrm{MAX}\text{-}k\text{-}\mathrm{SUM}$.
\end{definition}

These syntactic constraints ensure that all aggregation and combination functions are non-decreasing with respect to their inputs. Thus, the following can be observed.

\begin{restatable}{proposition}{mongnns}
\label{prop:mon-gnns}
Every positive-weight GNN is a monotonic GNN.
\end{restatable}

The syntactic classes defined above satisfy the corresponding structural preservation properties.

\begin{restatable}{theorem}{GNNpreservation}
\label{th:mon-GNN-preservation}
Augmented MGNNs are preserved under embeddings, MGNNs are preserved under injective homomorphisms, and $\mathrm{MAX}$-MGNNs  are preserved under homomorphisms.
\end{restatable}

These preservation results can be understood intuitively as follows.
Monotonic layers ensure that increasing node features (with respect to the underlying order) or 
neighbourhood information cannot invalidate positive predictions. 
Injective homomorphisms preserve the structure of the graph without 
merging nodes, and therefore preserve such monotonic behaviour. In 
contrast, non-injective homomorphisms may merge nodes and reduce 
multiplicities, which can invalidate properties that depend on 
counting. Restricting aggregation to $\mathrm{MAX}$ eliminates this issue, as $\mathrm{MAX}$ is insensitive to multiplicities: merging 
nodes does not affect the maximum value. Thus, 
$\mathrm{MAX}$-MGNN classifiers are preserved even under 
homomorphisms. Finally, augmentation layers allow for additional 
feature transformations applied pointwise, without changing the graph structure, thus preserving embeddings.

To obtain exact characterisations of GNNs classes, we start by 
establishing  lower bounds. That is, we show that every formula in a relevant fragment can be realised by a GNN in the corresponding syntactic class.

\begin{restatable}{theorem}{lowerBounds}
\label{thm:lowerBounds}
The following statements hold:
\begin{compactitem}
\item 
For each $\EGML$-classifier there is an  equivalent  augmented MGNN classifier, 
\item For each $\EML$-classifier there is  an equivalent  augmented $\mathrm{MAX}$-MGNN classifier,
\item For each $\EPGML$-classifier there is an equivalent  MGNN classifier, 
\item For each $\EPML$-classifier there is an  equivalent  $\mathrm{MAX}$-MGNN classifier. 
\end{compactitem}
Furthermore, each GNN above can be assumed to be simple, with truncated ReLU $\sigma(\mathbf{x}) = \min(\max(0,\mathbf{x}),1)$ as activation function.
\end{restatable}

To  derive exact characterisations of the expressiveness of GNN classes
we need to show matching upper bounds. For this, we use our preservation theorems together with the results in \Cref{sec:preservation}. Consequently we obtain the following tight bounds.

\begin{restatable}{theorem}{restMONMGNN}
\label{MONMGNN}
The following tight expressiveness results hold:
\begin{compactitem}
\item Augmented MGNN classifiers have the same expressiveness as $\EGML$ classifiers,
\item Augmented $\mathrm{MAX}$-MGNN classifiers have the same expressiveness as $\EML$ classifiers, 
\item MGNN classifiers have the same expressiveness as $\EPGML$ classifiers,
\item $\mathrm{MAX}$-MGNN classifiers have the same expressiveness as $\EPML$ classifiers.
\end{compactitem}
\end{restatable}

These results provide concrete insights into the kinds of properties
expressible by GNNs under different structural preservation 
constraints. At one end of the spectrum, augmented MGNNs, which 
correspond to preservation under embeddings, can express a wide range
of local properties, including pointwise non-monotone properties such as ``node $v$ does not have label $A$''. However, such models cannot express non-monotone properties that depend on the node's neighbourhood, such as exact counting (``exactly $k$ neighbours of $v$ satisfy property $p$''), parity conditions (e.g., ``the number of neighbours of $v$ is even''), or global properties such as connectivity.

Restricting to MGNNs, which correspond to preservation under injective homomorphisms, restricts expressiveness to monotone 
properties. These include existential and threshold-based conditions,
such as ``node $v$ has at least $k$ neighbours satisfying property 
$p$'', or ``there exists a path of length at most $k$ from $v$ to a node satisfying $p$''. 

At the most restrictive level, $\mathrm{MAX}$-MGNNs, which correspond to preservation under homomorphisms, can express
only purely existential properties, such as ``node $v$ has some 
neighbour satisfying property $p$''. In particular, they cannot 
express threshold-based conditions, as $\mathrm{MAX}$ aggregation is 
insensitive to multiplicities.

Overall, these results reveal a fundamental trade-off:  
preservation under stronger graph transformations leads to simpler logical 
characterisations, but at the cost of reduced expressive power. In 
particular, preservation properties constrain GNN expressiveness to 
fragments of modal logic, ruling out global and counting-based 
properties that fall outside these fragments.

\section{Conclusion}

We have developed a semantic framework for analysing the expressiveness of GNNs based on structural preservation properties. 
By combining preservation theorems with new well-quasi-order results for trees, we 
established exact correspondences between preservation under structural relations, fragments of graded modal logic, and natural classes of GNN architectures. 
Our approach highlights a fundamental trade-off between structural invariance and expressive power: stronger preservation properties lead to 
simpler logical characterisations, but at the cost of reduced expressiveness.
Our framework also extends naturally to directed graphs as the underlying logical results apply to Kripke models. 
The extension requires only minor adaptations to the GNN architecture, such as distinguishing between incoming and outgoing neighbourhoods.

\noindent
\textbf{Limitations }
Our framework focuses on node classification tasks; extending it to graph-level prediction remains an open problem.  Moreover, we study GNNs whose behaviour is determined by local neighbourhoods and structural preservation properties, and thus do not capture architectures that rely on global reasoning or other non-local mechanisms. 
We also restrict attention to threshold-based classifiers and to graphs without edge features.
Our logical characterisations assume discrete node labels represented as finite-dimensional binary vectors, and extending the results to richer feature domains is non-trivial. 
Finally, our correspondence with logic captures precisely the locality-bounded behaviour of GNNs, i.e., those invariant under bounded unravelling, and does not account for properties beyond this regime.

\bibliographystyle{abbrv}
\bibliography{biblio}


\appendix

\section*{Technical Appendix}

\section{Proofs for \Cref{sec:new_wqo}}

\embeddingwqo*
\begin{proof}
Let $L$ be the uniform bound on the height of the trees under consideration, and let $\preceq$ be the embedding relation between pointed trees.
Clearly, $\preceq$ is a quasi-order.
We prove by induction on $\ell\le L$ the following claim:
for every class $\T$ of pointed trees of height at most $\ell$,
the quasi-order $(\T,\preceq)$ is a wqo.

\smallskip
\noindent
\emph{Base case ($\ell=0$)}. A tree of height $0$ consists of a single root labelled by some vector in $\{0,1\}^d$. Since the dimension  $d$ of trees is finite,  there are finitely many non-isomorphic  trees of this form. 
Hence, for any class of pointed trees of height $0$, the quasi-order 
induced by $\preceq$ is a wqo.  Note  that the finiteness of 
$d$ is essential for this argument.

\smallskip
\noindent
\emph{Inductive step}. Let $\ell \geq 1$ and assume that the claim holds for $\ell-1$. 
Let $\T$ be an arbitrary class of pointed trees of height at most $\ell$. 
We show that $(\T, \preceq)$ is a wqo.  

We begin by partitioning $\T$
into subclasses $\T^S$ indexed by $S \subseteq \{1, \dots, d\}$, where $\T^S$ consists of all pointed trees in $\T$ whose root has exactly the labels in $S$.
Since $2^d$ is finite and a finite disjoint union of wqos is also a wqo,
it suffices to show that each $(\T^S, \preceq)$ is a wqo. 

Fix $S \subseteq \{1, \dots, d\}$. For each $(\M, w)$ in $\T^S$, let 
$\tau_{\M,w} = (\tau_1, \dots, \tau_n)$
be a finite sequence enumerating its immediate subtrees rooted at the children of $w$.
As the children are unordered, we fix an arbitrary but canonical ordering on the subtrees
to obtain a well-defined enumeration.
The embedding relation $\preceq$ on pointed graphs induces the subsequence order relation $\preceq^*$ on their finite sequences, defined  as in Higman's Lemma (see Lemma \ref{higman}).
To account for the arbitrary ordering of children, 
we consider the permutation-closed variant $\preceq^*_{\pi}$
defined by $\tau \preceq^*_{\pi} \tau'$ if and only if $\tau \preceq^* f(\tau')$ for some permutation $f$.
We claim that for any $(\M,w),(\N,v)\in\T^S$,
we have $(\M,w) \preceq (\N,v)$
if and only if $\tau_{\M,w} \preceq^*_\pi \tau_{\N,v}$.
Indeed, an embedding from $(\M,w)$ into $(\N,v)$ must map each child of $w$ injectively
to a distinct child of $v$, such that the corresponding subtree embeds recursively.
This yields $ \tau_{\M, w} \preceq^*_\pi \tau_{\N, v} $.
Conversely, if  $ \tau_{\M, w} \preceq^*_\pi \tau_{\N, v} $ then
each subtree in $\tau_{\M,w}$ embeds into a distinct 
subtree in $\tau_{\N,v}$, yielding an embedding from $(\M,w)$ into $(\N,v)$.

Thus, to show that $(\T^S, \preceq)$ is a wqo, it suffices to show that $\bigl( \{\tau_{\M,w} \mid (\M,w)\in\T^S\}, \preceq^*_\pi \bigr)$  is a wqo.
To this end, let $\mathcal I$ be the set of all rooted trees  of
height at most $\ell-1$.
By the inductive hypothesis,  $(\mathcal{I},\preceq)$ is a wqo. By  Higman's Lemma (Lemma \ref{higman}), the set of finite sequences $\mathcal{I}^*$ ordered by the subsequence relation $\preceq^*$ is also a wqo.
Since wqos are closed under taking subsets, and 
$(\tau_{\M,w} \mid (\M,w) \in \T^S)$ is a subset of $\mathcal I^*$,
it follows that $((\tau_{\M,w} \mid (\M,w) \in \T^S) , \preceq^*)$ is a wqo. Moreover, as $\preceq^*$ is a subrelation  of $\preceq^*_\pi$ and $\preceq^*_\pi$ is a quasi-order, it follows that $\preceq^*_\pi$ is also a wqo on this set. This completes the inductive step and the whole proof.
\end{proof}

\unboundedbad*
\begin{proof}
We start with the case of injective homomorphisms, which also covers embeddings as a special case.
Consider the class $\T$ of pointed trees that, for every $n\geq 1$, contains a rooted tree $(G_n, v_1)$ with nodes $v_1, \ldots v_n, v_{n+1}, v_{n+2},u,w$  which has edges  $\{ v_{i}< v_{i+1} \}$ for $1 \leq i \leq n+1$ as well as the edges
$\{ v_1, u \}$ and $\{ v_{n+1}, w\}$, as depicted in \Cref{antichain}. 
All nodes are unlabelled (equivalently, they all have the same label). 

In each $(G_n, v_1)$ there are exactly two nodes at distance $n+1$ from the root, namely 
$v_{n+2}$ and $w$. 
By contrast, if $m > n$, then $(G_m, v_1)$
has a single node at distance $n+1$, namely $v_{n+2}$.
Hence, for any $ n \neq m$, 
there is neither an injective homomorphism from $(G_n, v_1)$ to $(G_m, v_1)$, nor vice-versa.
Thus, $\T$ forms an infinite antichain with respect to injective homomorphisms
and hence also with respect to embeddings.

We now turn to the case of homomorphisms. Consider  the modification of $\T$ where, in each graph
$(G_n, v_1)$, nodes are labelled with vectors of length 4, where the 
labelling $\lambda$ is defined such that $\lambda(v)_1 =1$ iff  $v=v_i$ for some even $i$, 
$\lambda(v)_2 = 1$ iff $v=v_i$ for some odd $i$, $\lambda(v)_3=1$ iff $v=u$, and $\lambda(v)_4=1$ iff $v = w$. 

Let $n<m$ and suppose, towards a contradiction, that $f$ is a homomorphism  from 
$(G_n,v_1)$ to $(G_m,v_1)$.
We prove by induction on $i\le n+2$ that $f(v_i)=v_i$. 
The base case $i=1$ is immediate.
For the inductive step, assume that $f(v_i)=v_i$. 
Then, $f(v_{i+1})$ must be a successor of $v_i$ in $G_m$ satisfying the same labels as $v_{i+1}$ in $G_n$. 
In $G_m$ 
the only such successor  is $v_{i+1}$ itself, hence $f(v_{i+1})=v_{i+1}$.
In particular, $f(v_{n+1})=v_{n+1}$. It follows that $f(w)$ is a successor of $v_{n+1}$ in $G_m$  with label 4. 
However, the  only node 
in $G_m$ with label 4   is  $w$, which is a successor of
$v_{m+1}$ rather than $v_{n+1}$. This
yields a contradiction.  
A symmetric argument shows that no homomorphism exists from  $(G_m, v_1)$ to $(G_n, v_1)$.
\end{proof}

\section{Proofs for \Cref{sec:preservation}}

In what follows we will prove the following theorem.

\thpreserveunified*

The above theorem consists of three statements; we will prove them  one by one in three subsections.

Since proofs  exploit modal logic machinery, in the following subsections we will use notions and syntax used in this research area. In particular, instead of pointed graphs we will write about pointed (Kripke) models, and represent them as $\M=(W,R,V)$, where $W$ is a non-empty finite  set of \emph{worlds} (corresponding to nodes), $R \subseteq W \times W$ is the \emph{accessibility relation} (corresponding to edges), and $V$ is a \emph{valuation function} mapping each propositional variable $p_i \in \prop$ to a subset of $W$ (corresponding to the node labelling, where $\prop$ has as many labels as the dimension $d$ of vectors).

\subsection{Preservation Under Embeddings}

Our aim is to show that any class of models preserved under embeddings and invariant under $L$-unravelling is definable by a $\EGML$-formula. Moreover, we show
that such a formula can be chosen to have depth at most $L$.
To this end, we introduce characteristic formulae
in $\EGML$. These formulae are obtained  as a modification of the standard characteristic formulae for graded modal logic from the literature \cite{DBLP:journals/corr/abs-1910-00039,boundedGNNs}.
The key difference is that our characteristic formulae in $\EGML$ do not contain negated modal operators.

\begin{definition}\label{def:characteristic-existential}
Let $(\M,w)$ be a pointed model with $\M = (W,R,V)$.
For each $\ell \in \mathbb{N}$, we define
a \emph{characteristic $\EGML$ formula}, $\varphi_{\M,w}^\ell$  of depth $\ell$ inductively as follows.
For $\ell = 0$, let 
$$
\varphi_{\M,w}^0  := \bigwedge \{ p  \mid (\M,w) \models p \} 
\land
\bigwedge \{ \neg p  \mid (\M,w) \not\models p \}.
$$
For $\ell + 1$, partition the $R$-successors of $w$ into  equivalence classes $C_1, \dots, C_m$ according to the relation 
$\equiv^{\GML}_\ell$.
For each $C_j$, choose an arbitrary representative $w_j \in C_j$
and define 
$$
\varphi_{\M,w}^{\ell+1}
\ :=\
\varphi_{\M,w}^0 
\ \wedge\
\bigwedge_{j=1}^m \Diamond^{\geq |C_j|}\,\varphi_{\M,w_j}^\ell .$$
\end{definition}

For full \GML{}, a characteristic formula of depth $\ell$ for a pointed model
$(\M,w)$ holds in $(\N,v)$ if and only if there exists a graded $\ell$-bisimulation between  these two pointed models. Over
finite models, this is equivalent to $(\M,w) \equiv^\GML_\ell (\N,v)$ \cite{DBLP:journals/corr/abs-1910-00039}. 
The characteristic  defined above for \EGML{} are strictly weaker, and
these equivalences no longer hold.
Instead, the \EGML{} characteristic  ensure a one-way structural property: the $\ell$-unravelling of the  model used to construct the characteristic formula embeds into the $\ell$-unravelling of any  model satisfying the formula. 
This property is captured by the following lemma, which plays a central role in the 
preservation theorem for embeddings.

\begin{lemma}\label{lem:preserve}
Let $(\M, w)$  and $(\N, v)$ be pointed models and let  $\ell \in \mathbb{N}$.
We have $(\N,v) \models \varphi_{\M,w}^\ell$ 
if and only if $\unr^\ell(\M, w)$ embeds into $\unr^\ell(\N, v)$.
\end{lemma}
\begin{proof}
We prove the equivalence by induction on $\ell$. Let $\preceq$ be the
embedding relation between pointed models.

\smallskip
\noindent
\emph{Base case ($\ell=0$)}. Formula $\varphi_{\M,w}^0$ is a conjunction of all literals  (propositions and their negations) satisfied at $(\M,w)$. 
Thus, $(\N, v) \models \varphi_{\M,w}^0$ iff $(\M,w) \equiv^\GML_0 (\N, v)$.
Since $\unr^0(\M,w)$ and $\unr^0(\N,v)$ consist single worlds, it holds iff
 $\unr^0(\M, w) \preceq \unr^0(\N, v)$. 

\smallskip
\noindent
\emph{Inductive step.} Assume that the claim holds for $\ell \geq 0$. 
We prove that it also holds for $\ell+1$. 
Recall that
$
\varphi_{\M,w}^{\ell+1}
\ :=\
\varphi_{\M,w}^0 
\ \wedge\
\bigwedge_{j=1}^m \Diamond^{\geq |C_j|}\,\varphi_{\M,w_j}^\ell$, where 
$C_1, \dots, C_m$ are the equivalence classes of
$R$-successors of $w$
under
$\equiv^{\GML}_\ell$ , and  $w_j \in C_j$ is an
arbitrary representative. For each $j$, we also fix an
enumeration
$C_j = \{w_{j,1}, \ldots, w_{j,|C_j|} \}$.

\smallskip\noindent
\emph{($\Rightarrow$)}
Assume that $(\N,v) \models \varphi_{\M,w}^{\ell+1}$. 
Then, $(\N,v) \models \varphi_{\M,w}^{0}$, and hence $(\M,w) \equiv^\GML_0 (\N, v)$. 
Moreover, for each $j$, there exist $|C_j|$ distinct successors 
$v_{j,1}, \ldots, v_{j,|C_j|}$ of $v$ such that
$(\N, v_{j,k}) \models \varphi_{\M,w_j}^\ell$.

By the inductive hypothesis, for each such  $j$ and $k$, there exists an embedding $f_{j,k}$ from $\unr^\ell(\M, w_j)$ into $\unr^\ell(\N, v_{j,k})$. 
Since $w_{j,k} \equiv^{\GML}_{\ell} w_j$, the unravellings
$\unr^{\ell}(\M,w_{j,k})$ and $\unr^{\ell}(\M,w_j)$ are isomorphic.
Composing this isomorphism with $f_{j,k}$ yields an embedding of
$\unr^{\ell}(\M,w_{j,k})$ into $\unr^{\ell}(\N,v_{j,k})$.

We now construct a global embedding $f$ 
from $\unr^{\ell+1}(\M, w)$  into  $\unr^{\ell+1}(\N, v)$  as follows. 
Function $f$ maps the root of $\unr^{\ell+1}(\M, w)$ to the root of $\unr^{\ell+1}(\N, v)$. For each $j$ and $k$, $f$ coincides with the embedding constructed above on the 
subtree rooted at the node corresponding to $w_{j,k}$.
Since the successors $v_{j,k}$ are pairwise distinct, the images of these embeddings are
disjoint. It follows that $f$ is   an embedding of $\unr^{\ell+1}(\M,w)$ into $\unr^{\ell+1}(\N,v)$.

\smallskip
\noindent
\emph{($\Leftarrow$)}
Conversely, assume that $f$ is an embedding of  $\unr^{\ell+1}(\M, w)$ into $\unr^{\ell+1}(\N, v)$. 
Then, $f$ maps the root of $\unr^{\ell+1}(\M,w)$ to the root of $\unr^{\ell+1}(\N,v)$, and therefore, $(\M,w)\equiv^{\GML}_0(\N,v)$.
Thus,  $(\N, v) \models \varphi_{\M,w}^0$.
Fix $j \in \{1, \ldots, m\}$. 
By  the definition of unravelling,
the root of $\unr^{\ell+1}(\M,w)$ has exactly $|C_j|$
distinct nodes at depth $1$ corresponding to the successors
$w_{j,1},\dots,w_{j,|C_j|}$ of $w$.
Since $f$ is injective, it maps these nodes 
to $|C_j|$ distinct nodes at depth $1$ in $\unr^{\ell+1}(\N,v)$.
Each such  node  corresponds to a 
successor
$v_{j,k}$ of $v$ in $\N$, and the subtree below this node is precisely
$\unr^\ell(\N,v_{j,k})$.
Restricting $f$ to the corresponding subtrees yields embeddings, so
$\unr^\ell(\M,w_{j,k}) \preceq \unr^\ell(\N,v_{j,k})$
for $k=1,\dots,|C_j|$.
By the inductive hypothesis, this implies $(\N, v_{j,k}) \models \varphi_{\M,w_{j,k}}^\ell$.
Since each $w_{j,k}  \in C_j$, we obtain that $\varphi_{\M,w_{j,k}}^\ell$ is identical to $\varphi_{\M,w_{j}}^\ell$.
Moreover, the worlds $v_{j,k}$ are pairwise distinct successors
of $v$. Hence,
$(\N,v) \models 
\Diamond^{\geq |C_j|}\,\varphi_{\M,w_j}^\ell$, as required.
\end{proof}

Equipped with \Cref{lem:preserve} and \Cref{embedding_wqo}, we are  now ready to prove the preservation theorem for classes of models preserved under embeddings and invariant under $L$-unravelling.
As we show below, these are exactly the classes definable by \EGML{}- of modal depth  at most $L$.

\begin{theorem}\label{th:preserve}
The following are equivalent for any class  $\C$ of pointed models and any $L\in \mathbb{N}$:
\begin{itemize}
\item $\C$ is invariant under $L$-unravelling and preserved under embeddings,
\item $\C$ is definable by an $\EGML$ formula of depth at most $L$.
\end{itemize}
\end{theorem}

\begin{proof}

\smallskip
\noindent
\emph{($\Rightarrow$)}
Let $\T = \{ \unr^L(\M, w) \mid (\M,w) \in \C\}$ and let $\preceq$ be the embedding relation. 
Since $\C$ is invariant under
$L$-unravelling, we have that $\T \subseteq \C$.
By  \Cref{embedding_wqo}, $(\T, \preceq)$ is a wqo and therefore admits
only finitely many $\preceq$-minimal elements up to isomorphism. Let \Tmin contain exactly
one representative for each such minimal isomorphism type. 
For each $(\mathfrak T,v) \in \Tmin$, we construct the $\EGML$ characteristic formula $\varphi_{\mathfrak T,v}^L$ and define 
$\varphi_{\C} = \bigvee_{(\mathfrak T,v) \in \Tmin} \varphi_{\mathfrak T,v}^L$. 
Since the  disjunction is finite and each $\varphi_{\mathfrak T,v}^L$ is an \EGML{}-formula, $\varphi_{\mathcal{C}}$ is a (finite)  \EGML{}-formula. 
We show that $\varphi_{\C}$ defines $\C$.

Let $(\M, w) \in \C$.  Then  $\unr^L(\M, w) \in \T$. Since 
$(\T,\preceq)$ is a wqo, there exists a $\preceq$-minimal element $(\mathfrak T,v) \in \Tmin$ such that $(\mathfrak T,v) \preceq \unr^L(\mathfrak M, w)$.
Because $(\mathfrak T,v)$ is tree-shaped of height at most $L$, we have $\unr^L(\mathfrak T,v) = (\mathfrak T,v) $.
Hence, $\unr^L(\mathfrak T,v) \preceq \unr^L(\mathfrak M, w)$ and 
by \Cref{lem:preserve} we obtain  $(\M, w) \models \varphi_{\mathfrak T,v}^L$.
It follows that $(\M,w) \models \varphi_{\C}$.

Conversely, suppose that $(\M,w) \models  \varphi_{\C}$. Then, $(\M,w) \models  \varphi_{\mathfrak T,v}^L$
for some $(\mathfrak T,v) \in \Tmin$. 
By \Cref{lem:preserve}, this implies  $\unr^L(\mathfrak T,v) \preceq \unr^L(\M, w)$.
Since $(\mathfrak{T},v)$ is tree-shaped of height at most $L$, we again have
$\unr^L(\mathfrak{T},v)=(\mathfrak{T},v)$, and hence
$(\mathfrak{T},v)\preceq \unr^L(\M,w)$.
By construction, $\Tmin \subseteq \T \subseteq \C$, so
$(\mathfrak T,v) \in \C$. Since $\C$ is preserved under embeddings, we have $\unr^L(\M, w) \in \C$.
Finally, invariance of $\C$ under $L$-unravelling yields $(\M, w) \in \C$.

\smallskip
\noindent
\emph{($\Leftarrow$)} Suppose that $\C$ is definable by an $\EGML$-formula $\varphi$ of
depth at most $L$. Then $\C$ is invariant under $L$-unravellings: for every pointed model,
$(\M,w)\models\varphi$ if and only if $\unr^L(\M,w)\models\varphi$
since the truth of $\varphi$ depends only on the unravelling up to depth $L$.
We argue that $\C$ is also preserved under embeddings.
Let $(\M,w)\models\varphi$ and suppose that there exists an embedding
witnessing $(\M,w)\preceq(\N,v)$. Since embeddings preserve propositional labels and map successors injectively, a routine induction on the structure of $\varphi$ shows that $(\N,v) \models \varphi$. In particular, for each subformula of the form $\Diamond^{\ge k}\psi$ occurring in $\varphi$, 
the existence of $k$ successors satisfying $\psi$ in $(\M,w)$ implies the existence of
$k$ successors satisfying $\psi$ in $(\N,v)$. Thus, $\C$ is preserved under embeddings.
\end{proof}

\subsection{Preservation Under Injective Homomorphisms}

We next characterise unravelling-invariant  classes  that are preserved under injective homomorphisms.
Recall that every embedding is an injective homomorphism, but not conversely. 
This suggests that the logical language defining classes preserved under injective homomorphisms should be 
strictly weaker than the one required for preservation under embeddings.
By \Cref{th:preserve}, the latter language is \EGML.
As we show next, the appropriate language in the case of injective homomorphisms is \EPGML{}, the existential-positive fragment of \GML{}. 

We first define characteristic $\EPGML$-,  obtained from the characteristic  \EGML{}  by deleting all negative literals.

\begin{definition}
Let $(\M,w)$ be a pointed model with $\M = (W,R,V)$. The \emph{characteristic $\EPGML$-formula} $\psi_{\M,w}^{\ell}$ of depth $\ell$  is defined as the formula obtained from the
characteristic $\EGML$-formula $\varphi_{\M,w}^{\ell}$
(see \Cref{def:characteristic-existential}) by removing all negative literals $\neg p$.
\end{definition}

Note that  $\psi_{\M,w}^{\ell}$ is defined  in the same way as
$\varphi_{\M,w}^{\ell}$, except that the conjunction
$\bigwedge \{ \neg p \mid (\M,w)\not\models p \}$
is omitted.
As a result, $\psi_{\M,w}^{\ell}$ contains no negations and is therefore an \EPGML{}-formula.

While $\varphi_{\M,w}^{\ell}$ guarantees the existence of an embedding between $\ell$-unravellings of models, the formula $\psi_{\M,w}^{\ell}$ guarantees the existence of an injective homomorphism between these unravellings. This is captured by the following lemma.

\begin{lemma}\label{lem:preserve-injective}
Let $(\M, w)$  and $(\N, v)$ be pointed models and let  $\ell \in \mathbb{N}$. Then,
$(\N,v) \models \psi_{\M,w}^\ell$ 
if and only if there exists an injective homomorphism from $\unr^\ell(\M, w)$ to $\unr^\ell(\N, v)$.
\end{lemma}
\begin{proof}
The proof is by induction on $\ell$ and closely parallels the proof of Lemma~\ref{lem:preserve};
we highlight only the points where injective homomorphisms replace embeddings.
Throughout this proof, let $\preceq$
be the injective homomorphism relation.

\smallskip
\noindent
\emph{Base case ($\ell=0$).} The formula
$\psi^0_{\M,w}$ is the conjunction of all propositional symbols satisfied at $w$. Thus,
$(\N,v)\models\psi^0_{\M,w}$ if and only if 
$(\N,v)$ satisfies at least the same propositions as $(\M,w)$.
Since $\unr^0(\M,w)$ and $\unr^0(\N,v)$ each consist of a single world, this
holds if and only if 
$\unr^0(\M,w) \preceq \unr^0(\N,v)$.

\smallskip
\noindent
\emph{Inductive step.} 
Assume that the claim holds for some $\ell \geq 0$, and consider $\ell+1$.
We use the same notation as in the proof of \Cref{lem:preserve}. 

\smallskip
\noindent
\emph{($\Rightarrow$)}
Assume that $(\N,v) \models \psi^{\ell+1}_{\M,w}$. Then $(\N,v)\models\psi^0_{\M,w}$, and for each equivalence
class $C_j$, there exist $|C_j|$  distinct successors $v_{j,1},\dots,v_{j,|C_j|}$ of $v$ such that
$(\N,v_{j,k})\models\psi^\ell_{\M,w_j}$. By the inductive hypothesis, for each $j,k$,
there exists an injective homomorphism and hence $\unr^\ell(\M,w_j) \preceq \unr^\ell(\N,v_{j,k})$. As in the proof of 
\Cref{lem:preserve}, these homomorphisms can be combined together into a single injective homomorphism from $\unr^{\ell+1}(\M,w)$ to $\unr^{\ell+1}(\N,v)$  by
mapping the roots and using the homomorphisms above on the corresponding subtrees. 
Since the targets lie in pairwise disjoint subtrees rooted at distinct depth-$1$ nodes,
the resulting homomorphism is injective.
Thus, $\unr^{\ell+1}(\M,w) \preceq \unr^{\ell+1}(\N,v)$.

\smallskip
\noindent
\emph{($\Leftarrow$)}
Conversely, suppose that there exists an injective homomorphism
$f$ from $\unr^{\ell+1}(\M,w)$ to $\unr^{\ell+1}(\N,v)$. Then $f$ maps the root to the root, which implies $(\N,v)\models\psi^0_{\M,w}$.
Fix $j \in \{1,\ldots,m\}$. Injectivity of $f$ forces the $|C_j|$ distinct
depth-$1$ nodes below the root of $\unr^{\ell+1}(\M,w)$ to be mapped 
to $\mid C_j \mid$ distinct nodes of depth 1 of  $\unr^{\ell+1}(\N,v)$.
Each such node corresponds to a successor $v_{j,k}$ of $v$ in $\N$.
Restricting $f$ to the corresponding subtrees yields injective homomorphisms from $\unr^\ell(\M,w_{j,k})$ to $\unr^\ell(\N,v_{j,k})$. 
By the induction hypothesis we obtain $(\N,v_{j,k}) \models \psi^\ell_{\M,w_j}$ 
for all $k$. Since the worlds $v_{j,k}$ are pairwise distinct successors of $v$, we conclude that $(\N,v) \models \Diamond^{\ge |C_j|}\psi^\ell_{\M,w_j}$.
As this holds for each $j$, $(\N,v) \models \psi^{\ell+1}_{\M,w}$.
\end{proof}

Using the results above on characteristic $\EGML$-, we can now establish the following preservation theorem for unravelling-invariant classes preserved under injective homomorphisms.

\begin{theorem}\label{th:exists-positive-unravelling}
The following are equivalent for any class  $\C$ of pointed models and any $L\in \mathbb{N}$:
\begin{itemize}
\item $\C$ is preserved under injective homomorphisms and invariant under $L$-unravelling,
\item $\mathcal{C}$ is definable by an $\EPGML$-formula of depth at most $L$.
\end{itemize}

\end{theorem}
\begin{proof}

The proof follows the same strategy as that of Theorem \ref{th:preserve}, but now exploits \Cref{lem:preserve-injective} instead of \Cref{lem:preserve}.

\smallskip
\noindent
\emph{($\Rightarrow$)}
Let $\T=\{\,\unr^L(\M,w)\mid(\M,w)\in\C\,\}$ and
let $\preceq$ be the injective homomorphism relation.
Since $\C$ is invariant under $L$-unravellings, we have $\T \subseteq \C$. By  \Cref{cor:finiteness}, $(\T, \preceq)$ is a wqo and hence admits only finitely many $\preceq$-minimal elements up to isomorphism. Let \Tmin contain 
one representative for each such minimal isomorphism type. 
For each $(\mathfrak T,v) \in \Tmin$, we construct the $\EPGML$ characteristic formula $\psi_{\mathfrak T,v}^L$ and define 
$\psi_{\C} = \bigvee_{(\mathfrak T,v) \in \Tmin} \psi_{\mathfrak T,v}^L$. 
The argument showing that $\psi_{\C}$ defines $\C$ is identical to the corresponding one in the proof of  \Cref{th:preserve}, with \Cref{lem:preserve-injective} used in place of    \Cref{lem:preserve}.

\smallskip
\noindent
\emph{($\Leftarrow$)} Suppose that $\C$ is definable by an $\EPGML$-formula $\varphi$ of
depth at most $L$. Then $\C$ is invariant under $L$-unravellings: for every
pointed model  $(\M,w)$,  
$(\M,w)\models\varphi$ if and only if $\unr^L(\M,w)\models\varphi$. 
We now show that $\C$ is preserved under injective homomorphisms.
Let $(\M,w)\models\varphi$ and let $f$ be an
injective homomorphism
 witnessing $(\M,w)\preceq(\N,v)$. Since injective homomorphisms  preserve propositional labels monotonically and map successors injectively, a routine induction on the structure of $\varphi$ shows that $(\N,v) \models \varphi$, exactly as in the proof of Theorem \ref{th:preserve}.
\end{proof}

It is worth noting that logics corresponding
to both $\GML$ and $\EPGML$
have been studied extensively in the setting of description logics in knowledge representation and reasoning. 
In particular, the description logic corresponding to $\GML$ is $\ALCQ$ \cite{DBLP:conf/dlog/2003handbook}, while the 
logic corresponding to  $\EPGML$ is $\ELUQ$ \cite{morris2025sound}.

\subsection{Preservation Under Homomorphisms}

We now turn to classes preserved under homomorphisms.
We show that if such a class is invariant under unravelling, then it is definable by existential-positive modal logic ($\EPML$).
Our proof strategy follows the same general pattern
as in the cases of embeddings and injective homomorphisms.
In particular, we define an appropriate form of characteristic $\EPML$-.

\begin{definition}
Let $(\M,w)$ be a pointed model with $\M = (W,R,V)$, the \emph{characteristic $\EPML$-formula} $\chi_{\M,w}^\ell$ of depth $\ell$  is defined as the formula obtained from 
the characteristic $\EGML$-formula $\varphi_{\M,w}^{\ell}$
(see \Cref{def:characteristic-existential}) by removing all negative literals $\neg p$ and replacing each graded modal operator $\Diamond^{\geq k}$ with the basic modal operator $\Diamond$.
\end{definition}

Characteristic $\EPML$- are related to homomorphisms between unravellings by the following lemma.

\begin{lemma}\label{lem:preserve-homomorphism}
Let $(\M, w)$  and $(\N, v)$ be pointed models and let  $\ell \in \mathbb{N}$. Then,
$(\N,v) \models \chi_{\M,w}^\ell$ 
if and only if $\unr^\ell(\M, w) \preceq \unr^\ell(\N, v)$, where
$\preceq$ denotes the homomorphism relation.
\end{lemma}
\begin{proof}
The proof is by induction on $\ell$.

\noindent
\emph{Base case ($\ell=0$).}
The formula  
$\chi^{0}_{\M,w}$ is a conjunction of propositional
symbols satisfied at $w$. Thus, $(\N,v) \models \chi^{0}_{\M,w}$ if and only if $(\N,v)$ satisfies at least the same 
propositional symbols as $(\M,w)$. Since  $\unr^0(\M,w)$ and $\unr^0(\N,v)$ each consist of a single world, this holds if and only if 
$\unr^0(\M,w) \preceq \unr^0(\N,v)$. 

\noindent
\emph{Inductive step.} 
Assume that the claim holds for some $\ell \geq 0$, and consider $\ell+1$.
We use the same notation as in the proof of \Cref{lem:preserve}.

\noindent
\emph{($\Rightarrow$)}
Let $(\N,v) \models \chi^{\ell+1}_{\M,w}$. Then,
for each equivalence class $C_j$ of successors of $w$ in $\M$,
there exists at least one successor $v_{j}$ of $v$ in $\N$ such that
$(\N,v_{j})\models \chi^\ell_{\M,w_j}$. By the inductive hypothesis, for each $j$,
$\unr^\ell(\M,w_j) \preceq \unr^\ell(\N,v_{j})$. 
The unravelling $\unr^{\ell+1}(\M,w)$ consists of the root
together with subtrees rooted at depth-$1$ nodes corresponding
to the successors of $w$. Since homomorphisms need not be injective,
all depth-$1$ nodes corresponding to successors in the same class
$C_j$ may be mapped to the same depth-$1$ node corresponding to
$v_j$. Mapping the root to the root and using the homomorphisms above on the corresponding subtrees yields
a homomorphism from $\unr^{\ell+1}(\M,w)$ to $ \unr^{\ell+1}(\N,v)$.

\smallskip
\noindent
\emph{($\Leftarrow$)}
Let $f$ be a  homomorphism from  $\unr^{\ell+1}(\M, w)$ to $\unr^{\ell+1}(\N, v)$. Then, $f$ maps
the root to the root, implying $(\N, v) \models \chi_{\M,w}^0$.
Fix an equivalence class $C_j$. Every depth-$1$ node corresponding to a successor $w_j$ of $w$ must be mapped by $f$ to some
depth-$1$ node corresponding to a successor $v_j$ of $v$. Restricting
$f$ to the corresponding subtree yields a homomorphism
$\unr^{\ell}(\M,w_j)\to \unr^{\ell}(\N,v_j)$ and by the
inductive hypothesis we obtain $(\N,v_j)\models \chi^{\ell}_{\M,w_j}$.
Hence, $(\N,v)\models \Diamond \chi^{\ell}_{\M,w_j}$ for each $j$ and therefore $(\N,v)\models \chi^{\ell+1}_{\M,w}$.
\end{proof}

The preceding result allows us to establish the homomorphism preservation theorem for unravelling-invariant classes.

\begin{theorem}\label{th:exists-positive-unravelling-ML}
The following are equivalent for any class  $\C$ of pointed models and any $L\in \mathbb{N}$:
\begin{itemize}
\item $\C$ is preserved under homomorphisms and invariant under $L$-unravelling,
\item $\C$ is definable by an $\EPML$-formula of depth at most $L$.
\end{itemize}
\end{theorem}
\begin{proof}
The proof follows the  same general strategy
as that of \Cref{th:preserve}, but now uses  \Cref{lem:preserve-homomorphism} in place of \Cref{lem:preserve}.

\noindent
\emph{($\Rightarrow$)}
Let
$\T=\{\,\unr^L(\M,w)\mid(\M,w)\in\C \,\}$.
Since $\C$ is invariant under $L$-unravelling, we have
$\T\subseteq\C$.
Let $\Tmin\subseteq\T$ be a set of representatives of the 
homomorphism-minimal elements of $\T$. By  \Cref{cor:finiteness},
 $\Tmin$ is finite (alternatively, finiteness also follows from the fact that there are 
 only finitely many $\ML$-inequivalent  of depth at most $L$).
For each 
$(\mathfrak T,v)\in\Tmin$, let $\chi^L_{\mathfrak T, v}$ be its characteristic $\EPML$-formula and define 
$\chi_{\C} = \bigvee_{(\mathfrak T, v) \in \Tmin} \chi^{L}_{\mathfrak T,v}$, which is a finite $\EPML$-formula of depth at most $L$.

By construction, $\chi_{\C}$ holds in every model in $\Tmin$
and by \Cref{lem:preserve-homomorphism} it therefore holds in all models in $\T$. Since $\C$ is invariant under $L$-unravelling, it follows that $\chi_{\C}$ defines $\C$.

\noindent
\emph{($\Leftarrow$)}
Suppose that $\C$ is definable by an $\EPML$-formula $\varphi$ of depth at most $L$. Since $\varphi$ is positive, the
standard translation  yields an existential-positive first-order logic formula.
Such  are preserved under homomorphisms \cite{DBLP:journals/jacm/Rossman08}.
Moreover, since $\varphi$ has depth at most $L$, it is invariant under $L$-unravelling. 
\end{proof}

\section{Proofs for \Cref{sec:MGNN}}

\GNNpreservation*
\begin{proof}
We will show each of the three claims separately. We start from MGNNs, then consider  $\mathrm{MAX}$-MGNNs, and finally show the result for augmented  MGNNs.


\textbf{Claim 1.} MGNNs are preserved under injective homomorphisms.

Let $h \colon (G,v)\to(G',v')$ be an injective homomorphism between pointed graphs, and let $\GN$ be an  $L$ layer MGNN such that $\GN(G,v)=1$.
We will show that $\GN(G',v')=1$.

Since $\GN$ is monotonic, its classification function is non-decreasing, so it suffices to show that $\lambda^{(L)}(v) \leq \lambda^{(L)}(v')$. 
To this end, we prove by induction on $\ell \leq L$ that $\lambda^{(\ell)}(w) \leq \lambda^{(\ell)}(h(w))$, for every node $w$ in $G$. 

For $\ell = 0$, we observe that $h$ preserves node labels, so $\lambda^{(0)}(w) \leq \lambda^{(0)}(h(w))$. For the inductive step, assume the claim 
holds for $\ell < L$. Since $h$ is an injective homomorphism, it induces an injection from the multiset of 
neighbours of $w$ in $G$ to the multiset of neighbours of $h(w)$ in $G'$. By the induction hypothesis, this
yields $\lBrace \lambda^{(\ell)}(u)\mid u\in N_G(w)\rBrace \leq \lBrace \lambda^{(\ell)}(u')\mid u'\in N_{G'}(h(w))\rBrace$.
By monotonicity of the aggregation and combination functions, we conclude  that $\lambda^{(\ell+1)}(w)\le \lambda^{(\ell+1)}(h(w))$.

\textbf{Claim 2.} $\mathrm{MAX}$-MGNNs are preserved under homomorphisms.

The proof follows the same general structure as in the case of MGNNs.
Let $h \colon (G,v)\to(G',v')$ be a homomorphism between pointed graphs, and let $\GN$ be an $L$ layer $\mathrm{MAX}$-MGNN with $\GN(G,v)=1$.
We need to show that $\GN(G',v')=1$.
As before it suffices to prove by induction on $\ell \leq L$ that $\lambda^{(\ell)}(w) \leq \lambda^{(\ell)}(h(w))$, for each node $w$ in $G$. 

For $\ell = 0$ the inequality holds $h$, as a homomorphism, preserves node labels. For the inductive step, assume  the claim 
holds for $\ell < L$. Since $h$ is a homomorphism, each neighbour $u$ of $w$ in $G$ is mapped to a neighbour $h(u)$ of $h(w)$ in $G'$. By the induction hypothesis, $\lambda^{(\ell)}(u) \leq \lambda^{(\ell)}(h(u))$ and hence 
 $\mathrm{MAX}(\lBrace \lambda^{(\ell)}(u)\mid u\in N_G(w)\rBrace) \leq \mathrm{MAX}(\lBrace \lambda^{(\ell)}(u')\mid u'\in N_{G'}(h(w))\rBrace)$. By  monotonicity of the combination function, we conclude that  $\lambda^{(\ell+1)}(w)\le \lambda^{(\ell+1)}(h(w))$.
 

\textbf{Claim 3.} Augmented MGNNs are preserved under embeddings.

Let $h \colon (\G,v)\to(\G',v')$ be an embedding between pointed graphs,
and let $\GN$ be an $L$-layer augmented MGNN such that $\GN(\G,v)=1$. We will show that $\GN(\G',v')=1$.
To this end, it again suffices to prove by induction on $\ell \leq L$ that
$
\lambda^{(\ell)}(w) \leq \lambda'^{(\ell)}(h(w))$,  for every node $w$ in $\G$.



This time, however, layer 1 is different from  later layers, so we will consider it separately.
For $\ell = 0$, since $h$ is an embedding, it preserves node labels, and so,
$
\lambda^{(0)}(w) = \lambda'^{(0)}(h(w)).
$
For $\ell=1$, we have that the first layer is the augmented layer.
Since it is  applied pointwise, it follows that
$
\lambda^{(1)}(w) = \lambda'^{(1)}(h(w)).
$

For $\ell \geq 1$, the $\ell$th layer is monotonic.
Moreover $\lambda^{(1)}(w)$, for any node $w$, contains only non-negative numbers.
Therefore we can use the same argument as we used to show that MGNNs are preserved under injective 
homomorphisms. 
Thus, $
\lambda^{(\ell)}(w) \leq \lambda'^{(\ell)}(h(w))
$.
\end{proof}

Moreover, we can observe that MGNNs are not preserved under homomorphisms. To this end, let $G=(V,E,\emb)$ with $V=\{v,u,w\}$, $E = \{ \{v,u\}, \{v,w \} \}$, $\emb(v)=\emb(u)=\emb(w)=1$, and let  $G'=(V',E',\emb')$ with $V'=\{v',u' \}$, $E= \{ \{v',u' \} \}$, and $\emb'(v)=\emb'(u)=1$.
Define $\GN$ with a single layer, aggregation function $\mathrm{SUM}$, parameters $\mathbf{b}=0$ and $\mathbf{A}=\mathbf{C}=1$, and classification function
$\cls(x)=1$ iff $x \geq 3$.
Then,  $\GN(G,v)=1$, but $\GN(G',v')=0$.
However, the function $f$  defined by $f(v)=v'$, $f(u)=u'$ and $f(w) = u'$ is a (non-injective) homomorphism from $(G,v)$ to $(G',v')$.

\lowerBounds*


\begin{proof}
The statement consists of four claims, that we will prove one by one.

\textbf{Claim 1.} For each $\EGML$-classifier there is an  equivalent  augmented MGNN.

Fix $\varphi \in \EGML$ over propositions $\prop(\varphi)$ whose subformulae are $\mathsf{sub}(\varphi) = (\varphi_1, \ldots, \varphi_L)$, where $k \leq \ell$ whenever $\varphi_k$ is a subformula of $\varphi_{\ell}$.
We construct an equivalent augmented  MGNN  $\GN_{\varphi}$ with layers $0, \ldots, L+1$.

The first layer ($\ell = 0)$ is an augmentation layer which extends node features so that, for each proposition $p \in \prop(\varphi)$, both $p$ 
and its negation $\neg p$ are explicitly represented. Concretely, if the
original node labels 
are vectors in $\{0,1\}^d$, then the preprocessing layer maps each node 
$v$ to a vector $\emb^{(0)}(v)$ that contains, for each $p \in \prop(\varphi)$, a coordinate for $p$ with the same value as in the initial vector and a coordinate for $\neg p$, 
where the latter is computed as $1$ minus the value of the coordinate for $p$. Thus, the augmentation layer doubles the dimension of the input vector and the vectors it computes contain only non-negative entries. Note that the augmentation layer can be implemented as a pointwise affine transformation.

The remaining $L$ layers correspond to the subformulae of $\varphi$ and are specified below. The classification function $\cls$ is defined by $\cls(\mathbf{x}) = \true$ iff the last component of $\mathbf{x}$ is $1$.

Each layer $1, \ldots, L+1$ is simple (as Equation \eqref{eq:simple}):
\begin{equation} 
\emb'(v) = \sigma \Big(
\mathbf{b} +
     \emb(v) \mathbf{A}  +
         \agg \big( \lBrace  \emb(w) \mid w \in N_{\G}(v) \rBrace \big) \mathbf{C}  \Big),
\end{equation}
with 
 activation function $\sigma$ being  the componentwise truncated ReLU $\sigma(\mathbf{x}) = \min(\max(0,\mathbf{x}),1)$ 
and the aggregation function $\agg$ being  $\mathrm{MAX}$-$n$-$\mathrm{SUM}$, where $n$ is the largest integer occurring in a graded modalities of $\varphi$.

The entries of $\mathbf{A}, \mathbf{C} \in \mathbb{R}^{L \times L}$ and $\mathbf{b} \in \mathbb{R}^L$ are defined as follows:

\begin{center}
\begin{tabular}{l}
1. if $\varphi_\ell$ is a proposition or the negation of a proposition, then $C_{\ell\ell} = 1$; \\
2. if $\varphi_\ell = \varphi_j \wedge \varphi_k$, then $C_{j\ell} = C_{k\ell} = 1$ and $b_\ell = -1$; \\
3. if $\varphi_\ell = \varphi_j \vee \varphi_k$, then $C_{j\ell} = C_{k\ell} = 1$ and $b_\ell = 0$; \\
4. if $\varphi_\ell = \Diamond^{\geq c} \varphi_k$, then $A_{k\ell} = 1$ and $b_\ell = -c + 1$.
\end{tabular}
\end{center}

All other entries are set to $0$.

Observe that all entries of $\mathbf{A}$ and $\mathbf{C}$ are non-negative, so $\GN_{\varphi}$ is a positive-weight GNN with additional augmented layer. Hence, by \Cref{prop:mon-gnns}, $\GN_{\varphi}$ is an augmented MGNN, as required. 

We now show correctness of the construction. Consider the application of $\GN_\varphi$ to a graph $\G = (V,E,\emb)$. We show that for each subformula $\varphi_\ell$, each $i \in \{\ell, \ldots, L\}$, and each node $v \in V$,
\[
(\G,v) \models \varphi_\ell \quad \text{iff} \quad \emb^{(i)}(v)_\ell = 1.
\]

This implies that $\GN_\varphi(\G,v) = \true$ iff $(\G,v) \models \varphi$. The proof is by induction on the structure of $\varphi_\ell$.

\begin{itemize}
\item If $\varphi_{\ell}$ is a proposition then, $\lambda(v)^{(0)}_\ell =1$ if $(\G,v) \models \varphi_\ell$, and otherwise $\lambda(v)^{(0)}_\ell =0$.
Moreover, since $\varphi_\ell$ is a proposition,
all aggregate-combine layers have $C_{\ell \ell} = 1$, $b_{\ell} = 0$, and $A_{k\ell} = 0$ for each $k$, so
$\comb(\mathbf{x}, \mathbf{y})_{\ell} = \mathbf{x}_{\ell}$.
Thus $\lambda(v)^{(0)}_\ell = \lambda(v)^{(i)}_\ell$ for any $i$.

\item If $\varphi_{\ell}$ is the negation of a proposition, then its value is  computed by the augmented layer. In particular,
$\lambda(v)^{(0)}_\ell = 1$ if $(\G,v) \models \varphi_\ell$, and otherwise $\lambda(v)^{(0)}_\ell = 0$.
As in the propositional case, the construction ensures that this value is preserved across layers, so
$\lambda(v)^{(0)}_\ell = \lambda(v)^{(i)}_\ell$ for all $i \ge \ell$.

\item If $\varphi_{\ell} = \varphi_j \wedge \varphi_k$, then by construction, we have $C_{j\ell} = C_{k\ell} = 1$, $b_\ell = -1$, and $A_{m\ell} = 0$ for each $m$.
Then by Equation~\eqref{eq:simple}, we have
\[
\emb(v)^{(i)}_\ell = \sigma(\emb(v)^{(i-1)}_j + \emb(v)^{(i-1)}_k - 1).
\]
By the inductive hypothesis,
$\lambda(v)^{(i-1)}_j = 1$ iff $(\G,v) \models \varphi_j$, and similarly for $\varphi_k$.
Hence,
$\emb(v)^{(i)}_\ell = 1$ iff $(\G, v) \models \varphi_j \wedge \varphi_k$, and otherwise $\emb(v)^{(i)}_\ell = 0$ for all $i \ge \ell$.

\item If $\varphi_{\ell} = \varphi_j \vee \varphi_k$, then by construction, we have $C_{j\ell} = C_{k\ell} = 1$, $b_\ell = 0$, and $A_{m\ell} = 0$ for each $m$.
Then by Equation~\eqref{eq:simple}, we have
\[
\emb(v)^{(i)}_\ell = \sigma(\emb(v)^{(i-1)}_j + \emb(v)^{(i-1)}_k).
\]
By the inductive hypothesis,
$\lambda(v)^{(i-1)}_j = 1$ iff $(\G,v) \models \varphi_j$, and similarly for $\varphi_k$.
Hence,
$\emb(v)^{(i)}_\ell = 1$ iff $(\G, v) \models \varphi_j \vee \varphi_k$, and otherwise $\emb(v)^{(i)}_\ell = 0$ for all $i \ge \ell$.

\item If $\varphi_{\ell} = \Diamond^{\geq c} \varphi_k$, then by construction, we have $A_{k \ell} = 1$, $b_{\ell} = -c+1$, and $C_{m \ell} = 0$ for each $m$.
Since $\agg$ is the $\mathrm{MAX}$-$n$-$\mathrm{SUM}$, by Equation~\eqref{eq:simple}, we have
\[
\emb(v)^{(i)}_\ell =
\sigma\!\left(\min\!\bigl(n, \sum_{w \in N_G(v)} \emb(w)^{(i-1)}_k\bigr) - c + 1\right).
\]
By the inductive hypothesis,
$\lambda(w)^{(i-1)}_k = 1$ iff $(\G,w) \models \varphi_k$.
Hence,
$\emb(v)^{(i)}_\ell = 1$ iff $(\G, v) \models \Diamond^{\geq c} \varphi_k$, and otherwise $\emb(v)^{(i)}_\ell = 0$.
\end{itemize}

\textbf{Claim 2.}  For each $\EML$-classifier there is  an equivalent  augmented $\mathrm{MAX}$-MGNN.

Recall that $\ML$ is a fragment of $\GML$ where all graded modalities have index  $1$. 
Then, the $\mathrm{MAX}$-$n$-$\mathrm{SUM}$ function in our construction is $\mathrm{MAX}$-$1$-$\mathrm{SUM}$, which is simply the component-wise $\mathrm{MAX}$ function.
Hence $\GN_\varphi$ becomes an augmented $\mathrm{MAX}$-MGNN, as required. 

Finally we consider two last claims together:

\textbf{Claim 3.} For each $\EPGML$-classifier there is an equivalent  MGNN.

\textbf{Claim 4.} For each $\EPML$-classifier there is an  equivalent  $\mathrm{MAX}$-MGNN. 

In both cases formulae are negation-free, so the augmented layer in the construction of
$\GN_\varphi$ can be omitted. Hence our construction gives rise to an MGNN or $\mathrm{MAX}$-MGNN, depending on the case.
\end{proof}

\restMONMGNN*
\begin{proof}
Construction of a required type GNN from formulae from a given modal logic follows directly from \Cref{thm:lowerBounds}.

For the opposite direction we need to show that for a GNN $\GN$ of a given type there exists an equivalent formula $\varphi$ from a particular modal logic.
Let $\C$ be the class of pointed graphs accepted by $\GN$.
By Theorem~\ref{th:mon-GNN-preservation}, $\C$ is
preserved under: 
embeddings if $\GN$  is an augmented MGNN, 
injective homomorphisms if $\GN$  is an  MGNN, and under homomorphisms if $\GN$  is a  $\mathrm{MAX}$-MGNN.
Thus, by \cite{DBLP:conf/iclr/BarceloKM0RS20} 
there exists a formula  $\varphi$ in the required fragment of modal logic.

It remains to consider the case when $\GN$ is an augmented  $\mathrm{MAX}$-MGNN, and show that it is equivalent to some \EML{} formula.
We observe that every $\mathrm{MAX}$-MGNN can be seen as a bounded AC-GNN with set-based aggregation, denoted as $\GNN{\s}{\ac}$, as defined in \cite{boundedGNNs}.
It has been shown that $\GNN{\s}{\ac}$s have the same expressive power as $\ML$ \cite{boundedGNNs}, so we obtain that $\GN$ is equivalent to some $\ML$ formula.
Furthermore, recall that  augmented MGNNs are preserved under embeddings (\Cref{th:mon-GNN-preservation}) so $\GN$, as a special case of an augmented MGNN is also preserved under embeddings.
Consequently, it allows us to apply a well-known Rosen's preservation theorem for $\ML$, stating that an \ML{} formula is preserved under embeddings if and only if it is equivalent to an \EML{} formula \cite[Proposition 6]{DBLP:journals/jolli/Rosen97}.  
As a result we obtain that $\GN$ is equivalent to an \EML{} formula.
\end{proof}

\end{document}